\documentclass[journal]{IEEEtai}

\usepackage[colorlinks,urlcolor=blue,linkcolor=blue,citecolor=blue]{hyperref}

\usepackage{color,array}
\usepackage{booktabs}
\usepackage{graphicx}
\usepackage{subfigure}
\usepackage{amsmath}
\usepackage{amssymb}
\usepackage{bbm}
\usepackage[sort,square,numbers]{natbib}
\usepackage{soul}
\newtheorem{theorem}{Theorem}
\newtheorem{lemma}{Lemma}
\newtheorem{corollary}{Corollary}
\newtheorem{definition}{Definition}
\newtheorem{proposition}{Proposition}

\newcommand{\summe}[2]{\sum_{{#1}}^{#2}}
\newcommand{\logit}[0]{\mathbb{L}}
\newcommand{\mlogit}[0]{\mathcal{L}}
\newcommand{\mpr}[0]{\mathcal{P}}
\newcommand{\partder}[2]{\frac{\partial #1}{\partial #2}}
\def\pr{\ensuremath\mathbb{P}}

\begin{document}

\title{Maximum Probability Theorem: A Framework for Probabilistic Machine Learning}

\author{Amir Emad Marvasti, Ehsan Emad Marvasti, Ulas Bagci, Hassan Foroosh
\thanks{ Amir Emad Marvasti is with the Department of Computer Science, University of Central Florida, Orlando, FL 32826 USA (e-mail: aemad@cs.ucf.edu).}
\thanks{Ehsan Emad Marvasti is with the Department of Computer Science, University of Central Florida, Orlando, FL 32826 USA (e-mail: eemad@cs.ucf.edu).}
\thanks{Ulas Bagci is with the Machine and Hybrid Intelligence Lab,
Department of Radiology, Feinberg School of Medicine, and Department of BME, School of Engineering, 
Northwestern University, IL, USA (e-mail: ulas.bagci@northwestern.edu).}
\thanks{Hassan Foroosh is with the Department of Computer Science, University of Central Florida, Orlando, FL 32826 USA (e-mail: hassan.foroosh@ucf.edu).}
\thanks{This paragraph will include the Associate Editor who handled your paper.}}

\markboth{Journal of IEEE Transactions on Artificial Intelligence, Vol. 00, No. 0, Month 2020}
{Marvasti \MakeLowercase{\textit{et al.}}: Maximum Probability Theorem}

\maketitle

\begin{abstract}
We present a theoretical framework of probabilistic learning derived from the \textit{Maximum Probability (MP) Theorem} shown in the current paper.
   In this probabilistic framework, a model is defined as an \textit{event} in the probability space, and a model or the associated \textit{event} - either the true underlying model or the parameterized model - have a quantified probability measure.
   This quantification of a model's probability measure is derived by the \textit{MP Theorem}, in which we have shown that an event's probability measure has an upper-bound given its conditional distribution on an arbitrary random variable.
   Through this alternative framework, the notion of model parameters is encompassed in the definition of the model or the associated \textit{event}.
   Therefore, this framework deviates from the conventional approach of assuming a prior on the model parameters.
   Instead, the regularizing effects of assuming prior over parameters are imposed through maximizing probabilities of models or according to information theory, minimizing the information content of a model. 
   The probability of a model in our framework is invariant to reparameterization and is solely dependent on the model's likelihood function.
   Also, rather than maximizing the posterior in a conventional Bayesian setting, the objective function in our alternative framework is defined as the probability of set operations (e.g. intersection) on the \textit{event} of the true underlying model and the \textit{event} of the model at hand. 
   Our theoretical framework adds clarity to probabilistic learning through solidifying the definition of probabilistic models, quantifying their probabilities, and providing a visual understanding of objective functions.
\end{abstract}

\begin{IEEEImpStatement}
The choice of prior distribution over the parameters of probabilistic machine learning models determines the regularization of learning algorithms in the Bayesian perspective.
The complexity in choice of prior over the parameters and the form of regularization is relative to the complexity of the models being used.
Thereby, finding priors for parameters of complex models is often not tractable.
We address this problem by uncovering MP Theorem as a direct consequence of Kolmogorov's probability theory.
Through the lens of  MP Theorem, the process of regularizing models is understood and automated.
The regularization process is defined as the maximization of the probability of the model. The probability of the model is understood by MP Theorem and is determined by the behavior of the model.
The effects of maximizing the probability of the model can be backpropagated in a gradient-based optimization process.
Consequently, the MP framework provides a form of black-box regularization and eliminates the need for case-by-case analysis of models to determine priors.
\end{IEEEImpStatement}


\section{Introduction}

\IEEEPARstart{A}{} central problem in probabilistic learning and Bayesian statistics is choosing prior distributions of random variables and subsequently regularizing models. 
The importance of prior distributions is studied in Bayesian statistical inference since the choice affects the process of learning.
However, the choice of prior distributions is not clearly dictated by the axioms of probability theory.
In current applications of the Bayesian framework, the choice of prior differs from case to case, and still in many practical scenarios the choice of prior is justified by experimental results or intuitions.
Observe that there has been substantial attempts to unify the choice of priors over random variables, e.g. Laplace's Principle of Indifference \cite{jaynes2003probability}, Conjugate Priors \cite{gelman2006prior,diaconis1979conjugate}, Principle of Maximum Entropy \cite{jaynes1968prior}, Jeffreys priors \cite{jeffreys1946invariant} and Reference Priors \cite{berger2009formal}.
The overall goal of the existing literature is to pinpoint a single distribution as a prior to unify and objectify the inference procedure.

\subsection{The Proposed Maximum Probability Approach}
Contrary to many classic problems of inference and statistics where a hidden parameter of interest needs to be estimated, machine learning does not necessarily follow this goal.
The relevant solution to many complex problems in machine learning is the final likelihood of observable variables.
The likelihood functions in many cases are not necessarily the familiar likelihood functions ( e.g. Bernoulli, Gaussian, etc.) and may take complex forms.
A good example of such complex likelihood functions is Neural Networks in the context of image classification \cite{krizhevsky2012imagenet}.
If such networks are viewed through the Bayesian perspective, the model is the conditional distribution of labels given input images.
Finding an analytical close form for the prior over the parameters of complex models is not currently practical and to the knowledge of the authors, an automatic and practical procedure for assuming prior over the parameters of arbitrary models is not known to date.

A different approach is considered in this paper, where instead of assuming a prior over parameters, we focus on the likelihood functions.
A preliminary approach is to assume a density over possible likelihood functions, but we can simplify the perspective further by paying attention to the probability space.
Different likelihood functions on random variable $V$ can be formalized as $P_{V|M_\theta}$, where $M_\theta$ is an \textit{event} in the underlying probability space.
In this perspective, $\theta$ is enumerating the events in the probability space and serves as a numerical representative, not necessarily a random variable.
Through the MP theorem proved in this paper, by having $P_V$ - prior over the observables - an upper bound on the probability of $M_\theta$ can be calculated .
Thereby, instead of considering the likelihood function directly, the underlying event $M_\theta$, with a quantifiable probability measure is considered.

Viewing models as events can be extended to the true underlying model. 
Consequently, the problem of learning the true underlying model reduces to maximizing the probability of similarity between the model and the true model.
Since event are sets, the similarity between the model and the true underlying model is defined through set operations.
As an example, the probability of the intersection of the parameterized model and the true underlying model could be maximized.
The probability of the intersection event - as an objective function - is maximized through tuning the parameters of the model.
In the case of intersection objective function, we show that maximizing the probability of the model regularizes the model.

From the probability theory perspective, MP Theorem and its consequences extends the ability of probability theory to assign probabilities to \textit{uncertain} outcomes of random variables. 
It is because uncertainty in outcomes of a random variable can be modeled with a conditional distribution of the random variable given some underlying event.
We consider the probability of uncertain observations to be the probability of the underlying event.
We show that considering probability upper bound as the probability of the underlying event is consistent with existing definitions.
In this paper, we only investigate finite-range observable random variables and leave the extension of the continuous random variables for future works.

In probabilistic machine learning, MP Theorem has the following desirable properties. 
\textbf{(i)} the complexity of choosing prior is relative to the complexity of the observable random variables.
As an example, in the MP framework, having a Bernoulli observable random variable, one needs to assume a prior over a set with $2$ elements.
Consequently, through MP theorem the probabilities of different Bernoulli models can be calculated.
In the conventional approach, it is required to determine a prior distribution over the \textbf{real-valued} parameters of a Bernoulli distribution.
The conventional perspective ends up determining a prior on a disproportionately more complex set.
\textbf{(ii)} the models in our framework are not mutually exclusive and can have non-empty intersections.
It is intuitive to think that two Bernoulli models with parameters $0.9$ and $0.8$ should be related and do not represent disjoint events.
Conversely, in the conventional perspective, two models with the parameters $0.9$ and $0.8$ are treated as disjoint events.
\textbf{(iii)} Given the prior over the observables, in the MP framework the probabilities of models are determined by the characteristics of the likelihood functions. As such, per case analysis of likelihood functions are not needed.

We start by presenting MP theorem and its properties in Section \ref{sec:mindivg}.
The proofs for all the theorems can be found in Appendix A.
The connection of MP theorem to existing definitions and its interpretations are discussed in Section \ref{sec:interpret}.
Finally, the Maximum Probability Framework and examples of objective functions are presented in Section \ref{sec:modelsoracle}.
The detailed derivations for Section \ref{sec:modelsoracle} are presented in Appendix B and Appendix C.
\subsection{Background}
 The background for this topic also known as Objective Bayesian is broad enough to restrict the authors to include many important works.
 We refer the reader to the comprehensive review papers by Kass and Wasserman \cite{kass1996selection} and Consonni et al. \cite{consonni2018prior}.
 Here, we briefly review some of the important works that most impacted the objective Bayesian topic.

\textbf{Laplace's principle} is one of the early works to define priors leading to assuming a uniform prior over the possible outcomes of a random variable.
The downfall of this principle is seen in the case of real-valued random variables which leads to an \textit{improper} prior.
Aside from the impropriety of the prior, the approach is not invariant to reparameterization.   

\textbf{Principle of Maximum Entropy (MAXENT)}, introduced by Jaynes \cite{jaynes1957information,jaynes1957information2,jaynes2003probability}, provides a flexible framework in which testable information can be incorporated in the prior distribution of a random variable.
The prior is obtained by solving a constrained optimization problem in which the prior distribution with the highest entropy is chosen subject to the constraint of (i) testable prior information (usually represented in form of expectations) (ii) the prior integrating to 1. 
The shortcoming of MAXENT is in the case of continuous random variables, where the maximization of entropy is shown to be the minimization of the Kullback-Leibler divergence between the prior distribution and a base distribution which is obtained by Limiting Density of Discrete Points\cite{jaynes1957information,marsh2013introduction}.
The choice of the base measure is a similar problem to that of choosing the prior and as Kass and Wasserman \cite{kass1996selection} point out this is a circular nature in finding the maximal entropy distribution.
Furthermore, Seidenfeld in \cite{seidenfeld1979not} puts forward an example where MAXENT is not consistent with Bayesian updating.
In short, choosing a prior with MAXENT and obtaining the posterior with Bayes rule given some observations is not necessarily the same as choosing the prior with MAXENT given similar observations. 

\textbf{Jeffreys prior} is a class of priors that are invariant under one to one transformations.
Laplace's principle for finite-range random variables could be seen as a form of invariance under the permutation group.
The only distribution that is invariant under permutation of the states of a finite-range random variable is the uniform prior.
The general form of Jeffreys rule for the prior $p(\theta)$ over a one dimensional real parameter $\theta$ and a given likelihood function is $p(\theta) \propto \frac{1}{\sqrt{\det(\mathcal{I}(\theta))}}$, where $\mathcal{I}$ is the Fisher Information.
Jeffreys prior is constructed as a function of the likelihood function and does not provide any guideline for choosing the likelihood function over the observable .
This property is present in the context of \textit{Reference Priors} introduced by Bernardo \cite{bernardo1979reference} and further developed in \cite{berger2015overall,berger2009formal,berger1992development,berger1992ordered}.
Reference priors in the case of one dimensional parameter and some regularity conditions coincide with Jefferey's general rule \cite{kass1996selection}.

\textbf{Reference priors} construct the priors by maximizing the mutual information of the observable variable and the hidden variable (parameter).
The solution density is not necessarily a \textit{proper} prior but the usage is justified by showing that posterior is the limit case of posteriors obtained from proper priors\cite{berger2009formal}.
The solution density similar to Jeffreys prior is invariant under one to one transformations, which follows from invariance of the mutual information.
The mutual information between the observable and the parameter is a function of the likelihood function.
Thereby reference priors similar to the case of Jeffreys priors are only dependent on the likelihood function and do not determine a clear guide to construct likelihood functions. 
\subsection{Notation}
We use $\pr, P_V$ for probability measure, and the probability distribution of some random variable $V$, respectively.
$\logit, L_V$ corresponds to the logarithms of the probability measure and log-probability distribution of some random variable $V$, respectively.
We use $R(V)$ to represent the range of some random variable $V$ and $|R(V)|$ is the cardinality of the range of $V$.
For simplicity of notation and depending on the context, we use $\pr(v)$ as a shorthand notation of $\pr(V^{-1}(v))$. The lower case letters represent the outcomes of the random variable with the corresponding uppercase letter.
\section{Maximum Probability Theorem}\label{sec:mindivg}
The following theorem is the foundation of our work, bounding probabilities of \textit{events} using their conditional distribution. 
\begin{theorem}\label{thrm:Master}
Consider the probability space $(\Omega,\Sigma,\pr)$, the random variable $V$ with finite range $R(V)$ and $P_V(.)$ the probability distribution of $V$. For any event $\sigma \in \Sigma$ with the conditional distribution $P_{V|\sigma}(.)$ the following holds,

\begin{align}
    \pr(\sigma) \leq\underset{v\in R(V)}{\inf} \left \{ \frac{P_V(v)}{P_{V|\sigma}(v) }\right \} \triangleq \mathcal{P}(\sigma \succ V)
\end{align}
$\mathcal{P}(\sigma \succ V)$ is read as maximum probability of $\sigma$ observed by $V$ and $-\mathcal{L}(\sigma \succ V)= -\log\left(\mpr(\sigma \succ V)\right)$ is read as the minimum information in $\sigma$ observed by $V$.
\end{theorem}
The proof for Theorem \ref{thrm:Master} is short and simple, yet it is fundamental in understanding probabilistic models.
In this view, every probability distribution over the random variable $V$ corresponds to an event where the probability of the event is bounded using Theorem \ref{thrm:Master}.
The bound for probability in Theorem \ref{thrm:Master} can be decreased by extending random variables (Definition \ref{def:extendvar}).
\begin{definition}\label{def:preimage}
The preimage of an outcome $v$ of the random variable $V$, denoted by $V^{-1}(v)$, is defined as
\begin{align}
    V^{-1}(v)\triangleq\lbrace \omega |\omega \in \Omega , V(\omega) = v\rbrace\newline
    \quad, \forall v \in R(V)
\end{align}
Note that since $V$ is a measurable function, then $V^{-1}(v) \in \Sigma$.
\end{definition}
\begin{definition}\label{def:extendvar}
A random variable $W$ \textit{extends} random variable $V$ \textbf{iff}
\begin{align}
    &W^{-1}(w) \subset V^{-1}(v)\quad \textrm{or}\quad W^{-1}(w) \cap V^{-1}(v) =\emptyset\newline
    \notag\\&\forall w \in R(W), \forall v \in R(V).
\end{align}
\end{definition}
A random variable can be extended by increasing the cardinality of its range.
The following theorem shows that the probability upper bound is decreased by extending the random variables.
\begin{theorem}\label{thrm:ExtendedInfo}
For any random variable $W$ that extends $V$, the following inequality holds
\begin{align}
\mathcal{P}(\sigma\succ W) \leq \mathcal{P}(\sigma \succ V)
\end{align}
\end{theorem}
The simplest example of extending a random variables is by including additional random variables to describe the underlying event.
\begin{corollary}\label{remark:concatinfo}
For any random variable $V$ and $Z$ with concatenation $H=(V,Z)$ the following holds
\begin{align}
\mathcal{P}(\sigma\succ H) \leq \mathcal{P}(\sigma \succ V)
\end{align}
\end{corollary}
Theorem \ref{thrm:ExtendedInfo} shows that the Maximum Probability bound is relative to the complexity of the random variable. 
In simple terms, the random variables are tools to observe the underlying event.
By extending a random variable, its number of states is increased and the underlying event is more specified.
Since the event is more specified, its probability decreases.
The upper bound in (\ref{thrm:Master}) follows a similar logic.
Therefore the upper bound of probabilities of events is relative to the characteristics of the random variables.
We delve into the meaning of the upper bound in the maximum probability theorem in the next section.

\subsection{Interpretation of The Probability Upper Bound}\label{sec:interpret}
The upper bound in MP theorem have a concrete connection to existing definitions related to random variables. 
The upper bound nature in MP theorem is embedded in existing definitions of the probability of outcomes of random variables. 
To demonstrate the former statement, we review the following existing definition.
\begin{definition}\label{def:classicvar}
Probability of an outcome $v$ of random variable $V$ denoted by $P_V(v)$ is defined as
\begin{align}
    P_V(v) = \pr(V^{-1}(v)), \quad v \in R(V)
\end{align}
\end{definition}
Definition \ref{def:classicvar} defines the probability of an outcome of $v \in R(V)$ to be the probability of the largest event in the sense of number of elements $V^{-1}(v)$, that is mapped to $v$.
It is possible to show an equivalent definition in the sense of probability.
\begin{proposition}\label{def:maxprobdef}
Probability of an outcome $v$ of random variable $V$ in the sense of maximum probability defined as
\begin{align}
    P^*_V(v) \triangleq \sup\left\lbrace\pr(\sigma) | \sigma \in \Sigma , \forall \omega \in \sigma, V(\omega) = v  \right\rbrace,\quad  v \in R(V)
\end{align}
has the following property
\begin{align}
P^*_V(v) = P_V(v)\quad \forall v \in R(V).    
\end{align}
\end{proposition}
\begin{figure*}
\center
\subfigure[The Generic Setup]{
        \centering
        \includegraphics[scale=0.1306]{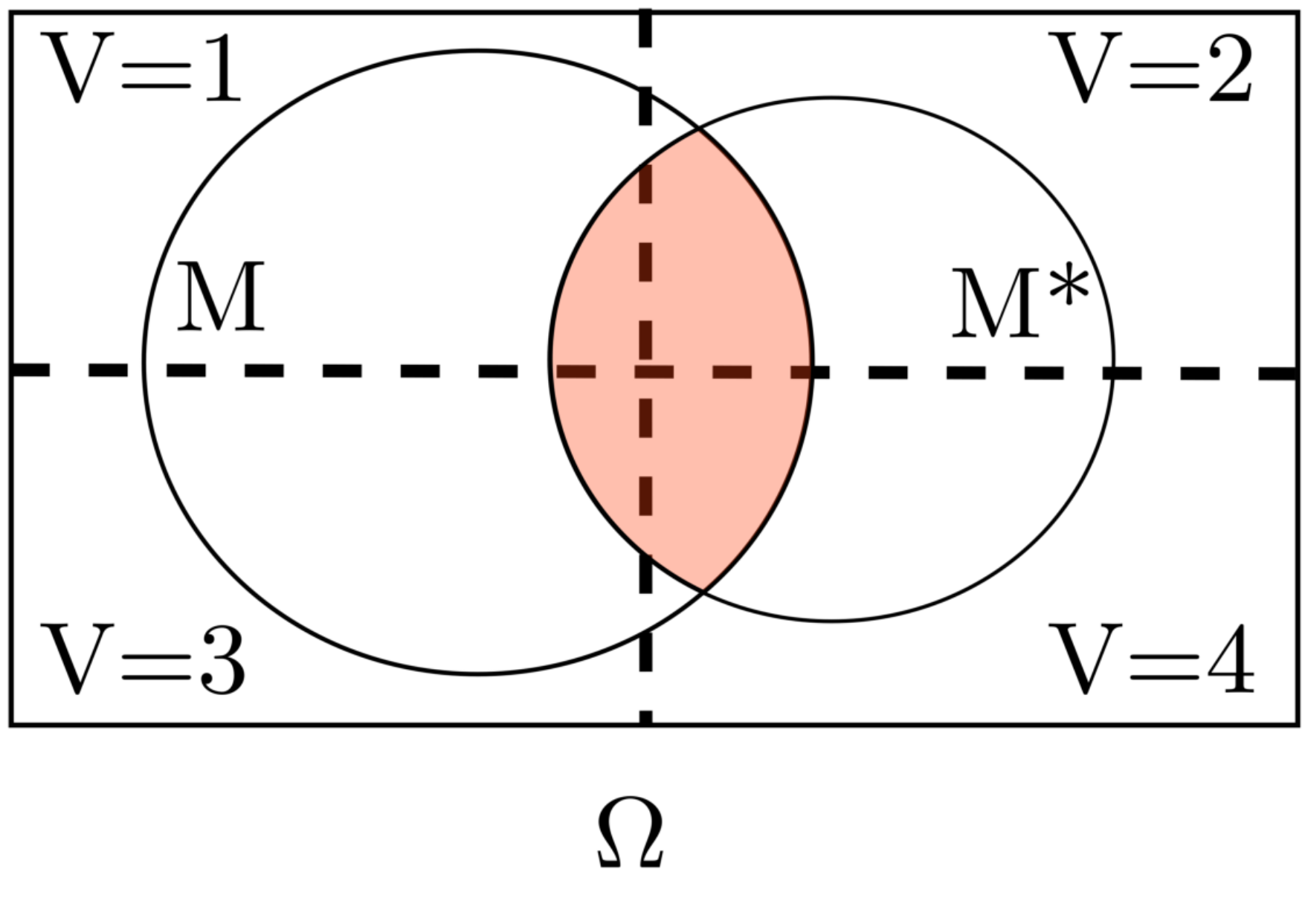}
        \label{fig:general}
        }
\subfigure[Likelihood Solution]{
        \centering
        \includegraphics[scale=0.1306]{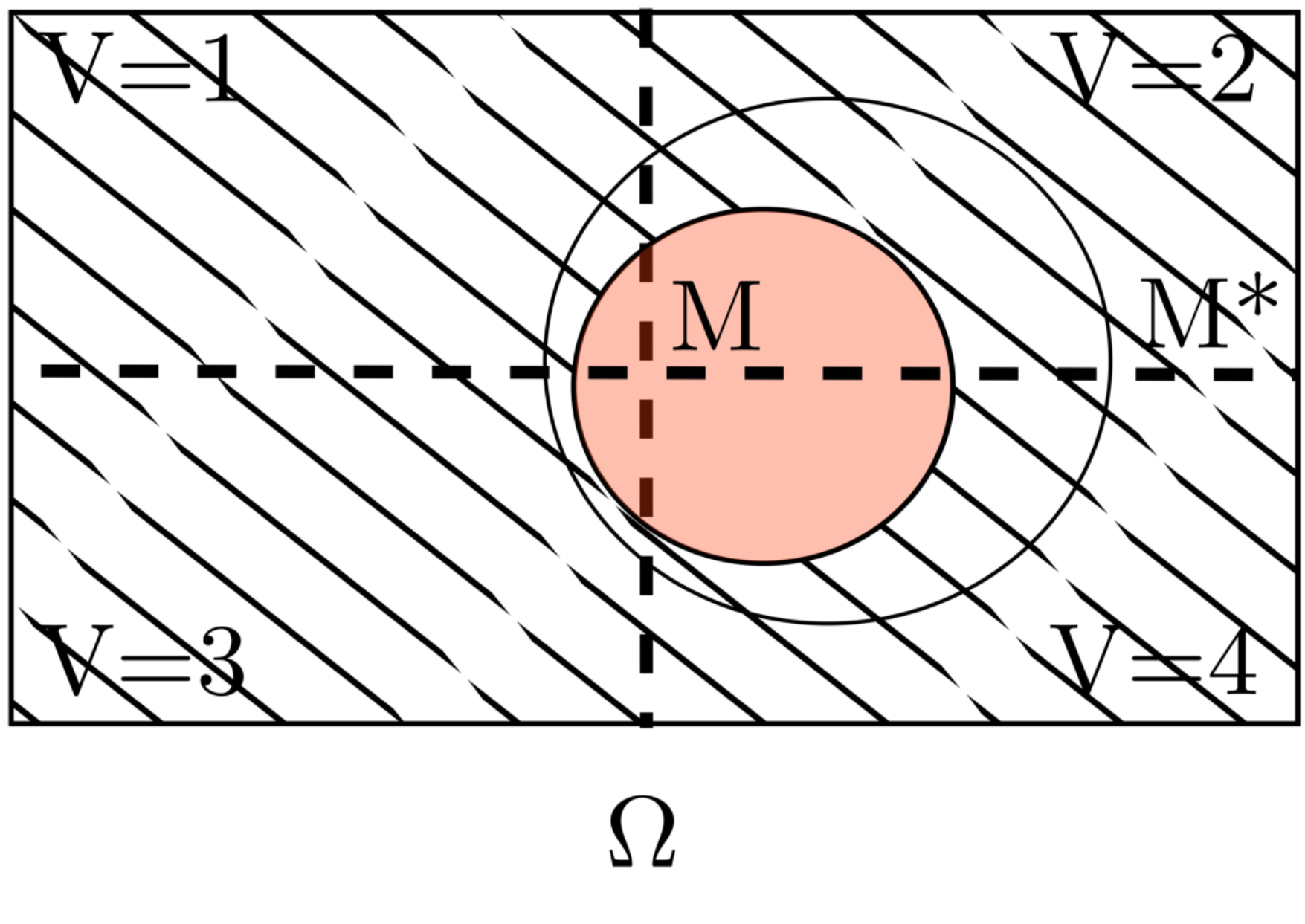}
        \label{fig:likelihoodsolution}
        }
\subfigure[Intersection Solution]{
        \centering
        \includegraphics[scale=0.1306]{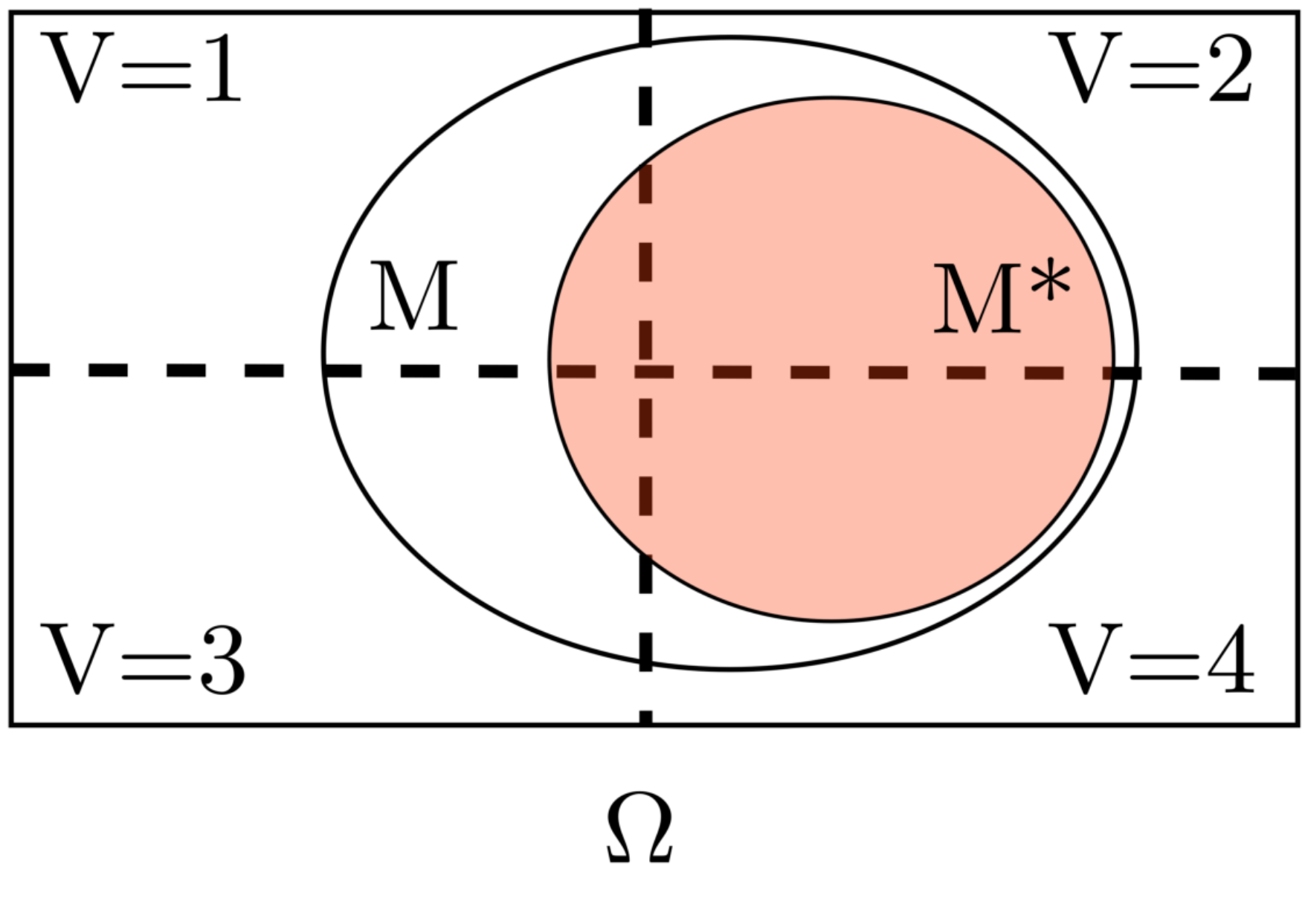}
        \label{fig:intersectionsolution}
    }       
\caption{The probability space in MP framework.
The model $M$ and the oracle $M^*$ are events in the probability space $(\Omega,\Sigma,\pr)$.
The dashed lines representing the random variable $V$ with $|R(V)|=4$, partitioning the probability space.
Probability of the model $M_\theta$ is dependent on $P_{V|M}$ and $P_V$ and is bounded by $\mathcal{P}(M\succ V)$. 
The parameters of the model $M$ is tuned to mimic the underlying model $M^*$.
\ref{fig:general} depicts the general setup with the red region representing $M\cap M^*$.
\ref{fig:likelihoodsolution} shows a candidate model $M$ achieving the global maximum $\pr(M^*| M)$.
The grey region in \ref{fig:likelihoodsolution} represents $\overline{M}$, which is normalized out in the likelihood objective function.
\ref{fig:intersectionsolution} shows a candidate model $M$ achieving the global maximum of intersection probability.
}\label{fig:delta}
\end{figure*}
Preimage of an outcome $v$ coincides with the most probable event being mapped to $v$ because of the monotonicity of the probability measure.
While many events can be mapped to the outcome $v$, both definitions consider the \textbf{largest} underlying event; either in the sense of cardinality or probability.
We can interpret the upper bound of probability in spirit of $P^*_V(v)$.
Given some information about the underlying event, we assume that the underlying event is the one with largest probability measure.
For example, having only the conditional probability distribution $P_{V|\sigma}$, we consider the upper bound as the probability of $\sigma$.
We can also use Information Theory to interpret the probability upper bound.
In Information Theory \cite{cover2012elements,mackay2003information,hansen2001model}, information content of an event $\sigma \in \Sigma$ is quantified as $-\logit(\sigma)$.
Information content of $\sigma$ is the minimum bits required to distinguish between $\sigma$ and its complement $\overline{\sigma}$.
As probability of $\sigma$ decreases, its information content increases and therefore requires further description.
Considering lower probability than the maximum bound is translated to more information content in an event; which is not based on the given information.
Considering lower probability may be interpreted as appending \textbf{assumptions} to the description of the observation.
\begin{corollary}\label{corollary:mprob_eqv}
Considering Definition \ref{def:classicvar} and Theorem \ref{thrm:Master}, given some set $\sigma_v\in \Sigma$ where $P_{V|\sigma_v}(v)=1$,
\begin{align}
P_V(v) = \mathcal{P}(\sigma_v \succ V)    
\end{align}
\end{corollary}
Corollary \ref{corollary:mprob_eqv} shows that $\mathcal{P}$ is equivalent to $P_V$ in the case of \textit{exact} observations of $V$.
We can define \textit{uncertain} observations as outcomes that are not completely determined.
Uncertain observations may be represented by a probability distribution, that is conditioned on some underlying event.
Exact observations may be considered as degenerate conditional distributions; special cases of uncertain observations.
Definition \ref{def:classicvar} fails to address the probability of uncertain observations of $V$, but $\mathcal{P}$ extends to uncertain observations.
We use an example to represent a scenario with uncertain observations and how MP theorem extends to such scenarios.
Imagine a fair coin is being flipped in a room. Alice asks Bob to investigate the room and tell her the outcome of the coin flip. Consider two scenarios:
    \textbf{(i)} Bob tells Alice that the coin is surely Head. Alice using existing definitions concludes that the probability of the event is $0.5$. The conclusion is similar if Alice uses $\mathcal{P}$ to calculate the probability. Alice can model the observation as a degenerate distribution conditioned on the event $\sigma_1$, i.e. $\pr(H|\sigma_1)$=1.
    The maximum probability of $\sigma_1$ is $\inf\lbrace \frac{0.5}{1},\frac{0.5}{0}\rbrace= 0.5$. 
    \textbf{(ii)} Bob observes some underlying event and using Bayes rule comes up with the conclusion that the coin is Head with probability $0.9$.
    Bob informs Alice about his conclusion.
    Alice cannot use existing definitions to calculate the probability of the underlying event.
    However using the Maximum Probability bound, she concludes that the probability of the evidence that Bob observed is at most $\inf\lbrace\frac{0.5}{0.9},\frac{0.5}{0.1}\rbrace=\frac{5}{9}$.
\section{Maximum Probability Framework}\label{sec:modelsoracle}
In the current trend of Bayesian statistics and machine learning, a parameterized family of distributions is assumed.
Subsequently, the underlying model is estimated by finding the parameter setting maximizing the posterior distribution, or creating an ensemble of models with parameters drawn from the posterior. 
Other alternatives include Variational Inference \cite{blei2017variational,ranganath2014black,jordan1999introduction}, where an approximation for the posterior is chosen from a parametric family. The posterior approximation is chosen by finding the distribution in the family with minimum Kullback-Leibler Divergence to the true posterior. All of the mentioned approaches treat the parameters as a random variable and require a prior distribution over the parameter space.

We introduce the Maximum Probability framework (MP framework) for probabilistic machine learning as a corollary of the MP theorem.
In MP Framework, we define a \textit{model} as an event, $M_\theta \in\Sigma$, where $\theta \in \mathbb{R}^d$ is the parameter, and the conditional distribution of the random variable $V$ given the model is $P_{V|M_\theta}$. 
Also, the true underlying model or \textit{oracle} is represented as $M^*\in \Sigma$, with the conditional distribution $P_{V|M^*}$.
$M_\theta$ and $M^*$ are not explicitly defined.
Instead, the conditional distribution of the model is obtained through some explicit and deterministic \textit{parameterization} function $\pi$, i.e. $P_{V|M_{\theta}}(v) = \pi(v,\theta)$. 
Also, the oracle can be understood as an observation from a generative process.
For example we can define the conditional distribution of oracle given some observation $v_o\in R(V)$, as $P_{V|M^*}(v_o)=1$.
Note that the random variable $V$ could be modeled as the concatenation of multiple i.i.d random variables $(V^{(i)})_{i=1}^N$.
In the i.i.d modeling case $V$ would represent a random variable corresponding to multiple observations, that conditioned on the oracle are independent.
The choice of modeling for the random variables is arbitrary and we focus on $V$ as a generic random variable.

Given a prior distribution $P_V$ and through maximum probability theorem, we can calculate the probability of models under different parameter settings. 
As opposed to the conventional approach where a probability \textbf{density} is assumed on the parameters - an uncountable infinite set - the prior is chosen over the finite range of $V$.
Furthermore, in the MP framework each model has a probability \textbf{mass} and models potentially have intersections in the underlying probability space.
The visual representation of the underlying probability space containing the model and the oracle is depicted in Fig \ref{fig:delta}.
The goal is to increase the similarity of the underlying events corresponding to the model and the oracle by tuning the parameters.
Since the model and the oracle are sets, we can construct objective functions using set operations between the model and the oracle.
We explore maximizing two objective functions to demonstrate properties of each, i.e. \textit{log-likelihood} and \textit{intersection} corresponding to $\logit(M^*|M)$ and $\logit(M^*\cap M)$.
The probability of the outcome of such set operations reflects the similarity between the model and the oracle.

\subsection{Log-Likelihood Objective Function}
In this section, we investigate maximization of the log-likelihood of the oracle given the model, or $\logit(M^*|M_\theta)$.
We start by representing $\logit(M^*|M_\theta)$ as the marginalization of the random variable $V$
\begin{align}
    \logit(M^*|M) = \log \Bigg(\summe{v\in R(V)}{}&\pr(M^*|v,M)\pr(v|M)\Bigg)\nonumber\\
    &-\log\Bigg(\summe{v'\in R(V)}{}\pr(v'|M)\Bigg).\label{eq:loglikelihoodmarginal}
\end{align}{}
Taking the partial derivative of the log-likelihood with respect to $\logit(v|M_\theta)$
\begin{align}
    \partder{\logit(M^*|M_\theta)}{\logit(v|M_\theta)} &= \frac{\pr(M^*|v,M_\theta)\pr(v|M_\theta)}{\pr(M^*|M_\theta)} - \pr(v|M_\theta)\nonumber\\&=\pr(v|M^*,M_\theta) - \pr(v|M_\theta) .\label{eq:indpt_likelihoodexplicitform}
\end{align}{}
and setting the partial derivative to zero, we obtain the condition
\begin{align}
    \pr(v|M^*,M_\theta) = \pr(v|M_\theta).\label{eq:loglikelihood_sufficient}
\end{align}{}
Considering the chain rule for gradients, any parameter setting $\hat{\theta}$, that satisfies the condition in (\ref{eq:loglikelihood_sufficient}), is necessarily a critical point of the objective function.
The gradient of the log-likelihood objective function in (\ref{eq:indpt_likelihoodexplicitform}) is the subtraction of two probability vectors.
The stochastic nature of the gradient is suitable for stochastic gradient optimization of the parameters, especially for variables with large number of states.

Note that in (\ref{eq:loglikelihoodmarginal}), $\pr(M^*|v,M)$ is not determined by neither the conditional distribution of the model nor the oracle.
We assume that the oracle and the model given any outcome of the variable $V$ are conditionally independent, i.e. $\pr(M^*|v,M_\theta) = \pr(M^*|v)$.
The model and the oracle are defined implicitly through their conditional distribution on $V$.
Thereby it is natural in the problem of learning to assume that given the outcome of the random variable, $M$ does not convey information about $M^*$; hence the conditional independence.
Nevertheless, other alternatives can be assumed and derived separately. 
\begin{figure*}[t]
\center
\includegraphics[scale=0.45]{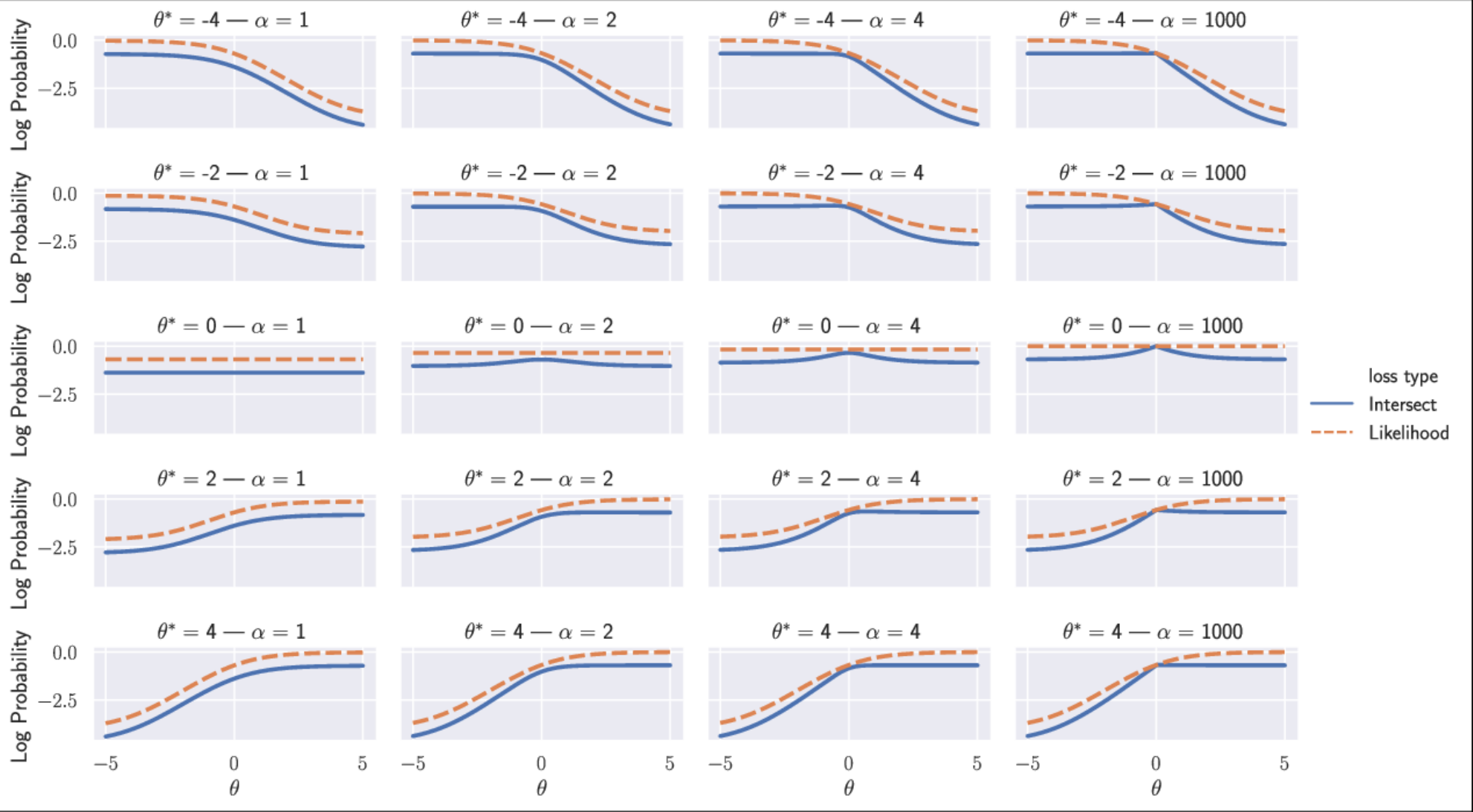}
\caption{Comparison of the objective functions in the Bernoulli example under conditional \textbf{independence} assumption. $\theta^*$ represents the oracle's log probability ratio, $\theta$ represents the model's log probability ratio.  In $\alpha=1$ the objective functions become similar. The intersection has a unique solution for $\alpha>1$. As $\alpha\to \infty$, the solution approaches the trivial solution and tends to lose its uniqueness. }\label{fig:lossindp}
\end{figure*}
The models maximizing the log-likelihood objective function have a special characteristic; their probability distribution is concentrated on the most probable outcome(s) of $P_{V|M^*}$.
The formal description and proof of this claim are brought in Appendix B.
The solution of optimization being a degenerate distribution is the analogue of overfitting in our framework. 
We show in the next section that the intersection objective function includes the log-probability of the model in the objective function.
Considering the model's log-probability, as will be shown, induces regularizing effects on the final solution by preventing the final distribution from being degenerate.
Thereby intersection objective function prevents the overfitting problem that occurs in the log-likelihood objective function.
\subsection{Intersection Objective Function}
Similar to our analysis of the log-likelihood objective function we focus on the properties of the solutions of intersection objective function.
The log-probability of the intersection of the model and the oracle can be written as 
\begin{equation}
    \logit(M^*,M_\theta) = \logit(M^*|M_\theta)+\logit(M_\theta)\leq \logit(M^*|M_\theta)+\mlogit(M_\theta\succ V).\label{eq:intersect1}
\end{equation}
The maximum probability of the model $\mlogit(M_\theta)$ in (\ref{eq:intersect1}), can be substituted with the logarithm of the so called \textit{softmax probability} family, $\mlogit_\alpha(M_\theta)$.
The following proposition defines the family and shows its important property.
\begin{proposition}\label{prop:alphafamily}
For the softmax probability family of functions $\mathcal{P}_\alpha$ defined as
\begin{align}
    &\mathcal{P}_\alpha(\sigma\succ V) \triangleq \left(\summe{v\in R(V)}{\space}\left(\frac{P_V(v)}{P_{V|\sigma}(v)}\right)^{-\alpha}\right)^{-\frac{1}{\alpha}},\quad \alpha >0,
    \end{align}
    the following is true,
    \begin{align}
    &\mathcal{P}_\alpha(\sigma\succ V)\leq \mathcal{P}(\sigma \succ V) ,\quad \forall \alpha> 0,\forall \sigma \in \Sigma,
\end{align}
and the equality holds as $\alpha \to +\infty$. The softmin information family is defined as
\begin{align}
-\mathcal{L}_{\alpha}(\sigma\succ V) \triangleq -\log\left(\mathcal{P}_\alpha(\sigma\succ V) \right).    
\end{align}
\end{proposition}
We define a flexible family of objective functions $\mlogit_\alpha(M^*,M)$ by substituting the softmin information family $\mlogit_\alpha$ instead of $\mlogit$ in (\ref{eq:intersect1})
\begin{equation}
 \logit_\alpha(M^*,M_\theta) \triangleq \logit(M^*|M_\theta) + \mlogit_\alpha(M_\theta\succ V).\label{eq:objectivefamilyintersect}   
\end{equation}
Without loss of generality we assume a uniform prior over $V$.
Since for all $v$, $P_V(v)=1/|R(V)|$ is a constant we can represent $\mlogit_\alpha(M^*,M_\theta)$ as
\begin{align}
    \mlogit_\alpha&(M^*,M_\theta) = \log\left(\summe{v\in R(V)}{}\pr(M^*,v|M_\theta)\right)\nonumber\\
    &- \frac{1}{\alpha}\log\left(\summe{v'\in R(V)}{}\pr(v'|M_\theta)^{\alpha}\right) + \frac{1}{\alpha}\log(|R(V)|),
\end{align}
Taking the partial derivative with respect to the log-probabilities we get
\begin{align}
    \partder{\mlogit_\alpha(M^*,M_\theta)}{\mlogit_\alpha(v|M_\theta)} &= \pr(v|M_\theta,M^*) - \frac{\pr(v|M_\theta)^\alpha}{\summe{v'\in R(V)}{}\pr(v'|M_\theta)^{\alpha}}.\label{eq:intersect2}
\end{align}
We call the second term in the right hand side of (\ref{eq:intersect2}) as $\alpha$-skeleton of a distribution for future references. 
\begin{definition}
\textbf{($\alpha$-skeleton) }given the conditional distribution $\pr(v|\sigma), \sigma\in\Sigma$, the $\alpha$-skeleton distribution of $V$ given $\sigma$ is defined as
\begin{align}
    \pr_\alpha(v|\sigma) \triangleq \frac{\pr(v|\sigma)^\alpha}{\summe{v'\in R(V)}{}\pr(v'|\sigma)^\alpha}, \quad \alpha \in \mathbb{R}.
\end{align}{}
\end{definition}{}
$\alpha$-skeletons concentrate probabilities of distributions depending on the sign and magnitude of $\alpha$.
In the limit case as $\alpha\to +\infty$ the $\alpha$-skeleton of a distribution concentrates the probability measure on the most probable outcome or outcomes. 
Setting the partial derivatives in (\ref{eq:intersect2}) to zero, we conclude the following sufficient condition for the critical points
\begin{align}
    \pr(v|M_\theta,M^*) =  \pr_\alpha(v|M_\theta).\label{eq:solintersect}
\end{align}
\begin{figure*}[t]
\center
\includegraphics[scale=0.45]{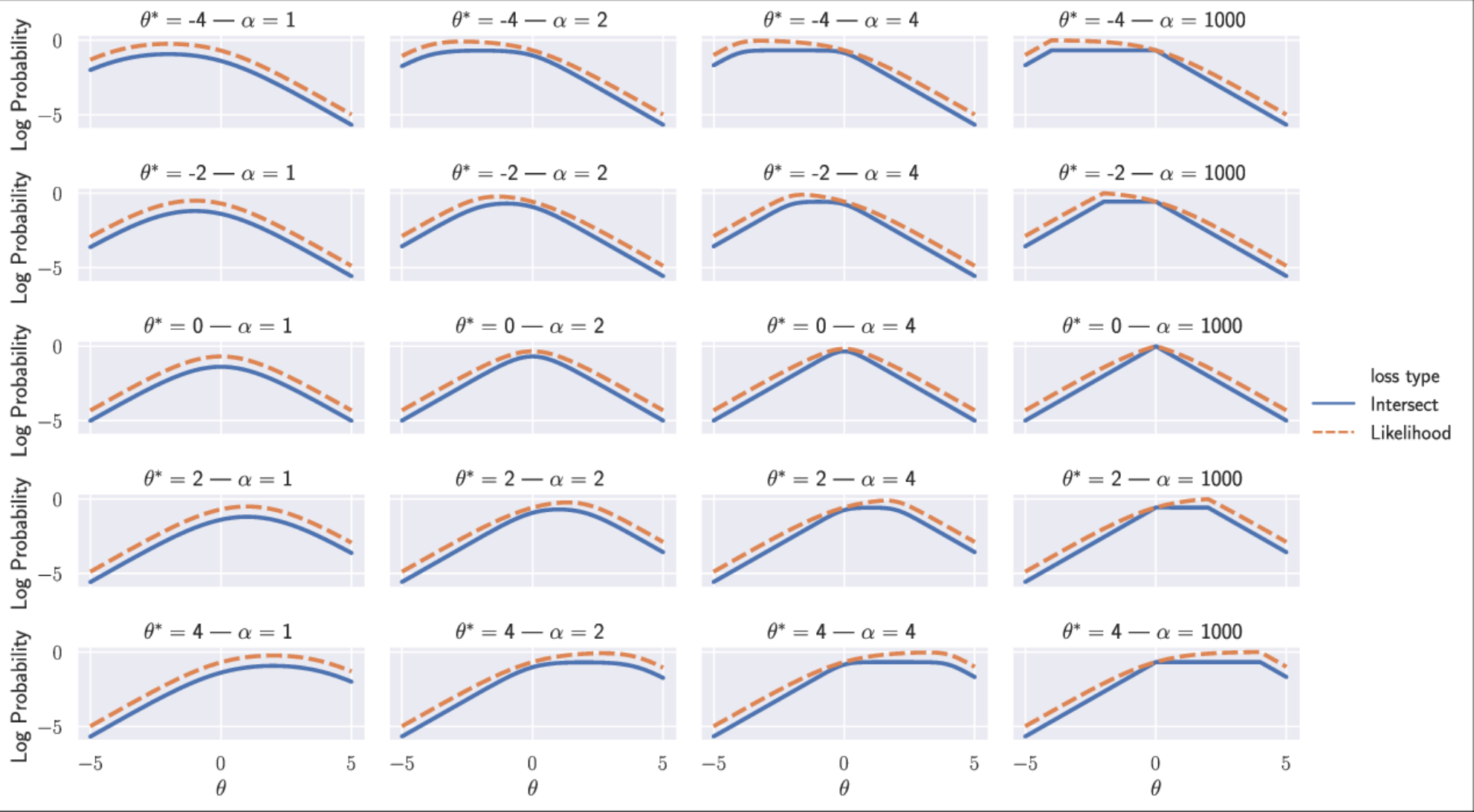}
\caption{Comparison of intersection and log-likelihood objective functions assuming that $M^*\subset M$. with respect to the parameters of the model for different values of $\alpha$. Note that $\theta^*= \logit(1|M^*)-\logit(0|M^*)$ and $\theta= \logit(1|M_\theta)-\logit(0|M_\theta)$. }\label{fig:LossSubset}
\end{figure*}
Equation (\ref{eq:solintersect}) is the sufficient condition and is satisfied when the $\alpha$-skeleton of the model distribution matches the posterior distribution.
The $\alpha$-skeleton term in (\ref{eq:solintersect}) acts as a regularizer suppressing the peaks in the model's distribution.
Similar to the log-likelihood case, both terms in (\ref{eq:intersect2}) are probability vectors and can be Monte Carlo approximated for stochastic optimizations.
Note that in the case of uniform priors and $\alpha=1$, the log-likelihood and the intersection objective functions are equal up to an additive constant. 
Setting $\alpha=2$ has a significant property that the solution \textbf{coincides} with oracle's likelihood function.
We can check the former claim by substituting $\pr(v|M_\theta)$ with $\pr(v|M^*)$ in (\ref{eq:solintersect}) (see Appendix C).
As $\alpha$ increases, the final solution is flatter than the solution of the likelihood objective function.
In the case of $\alpha \to +\infty$, the solution tends to the prior distribution $P_V$.
This is consistent with our visual representation in Fig \ref{fig:intersectionsolution}.
Setting $M_\theta =\Omega$ is the trivial solution for maximizing the intersection in the underlying space.
In general $\alpha$ as a hyperparameter, controls the regularization effect.
The usage examples of the MP Framework are presented in the next section.
\section{Examples}
In this section we present application of MP Framework on two examples: (i) A Bernoulli random variable (ii) Convolutional Neural Networks (CNNs).
In the simple case of Bernoulli random variable, we plot the objective functions to visualize their properties. Additionally we examine an alternative to the assumption that the model and the oracle are conditionally independent.
The alternative assumption is that the oracle is a subset of the model.
As a practical example we apply MP framework to CNNs.
we show that log-likelihood objective function and the conventional cross entropy loss are similar.
We further extend our derivation to intersection loss function and obtain a regularization term in the loss. 
The training of a CNN with intersection objective function for different values of $\alpha$ are presented and compared with the cross entropy loss with $\ell_2$ regularization.

\subsection{Bernoulli Random Variable}
 Consider the Bernoulli random variable $B$ with range $R(B) = \lbrace 0,1\rbrace$. Assume that the oracle $M^*$ has the log probability ratio $\theta^* = \logit(1|M^*)-\logit(0|M^*)$ and the model $M_\theta$ is parameterized by $\theta=\logit(B=1|M_\theta)-\logit(B=0|M_\theta)$. 
Note that $\theta$ and $\theta^*$ fully characterize the probability distributions using the sigmoid parameterization function, $\pr(1|M_\theta) = \frac{1}{1+e^{-\theta}}$.
We assume that $B$ has a uniform distribution over the sample space, i.e. $\pr(1)=0.5$. 
We can plot the behaviour of these objective functions with respect to the parameters $\theta$ and $\theta^*$.
Note that the only value that is not explicitly known is $\pr(M^*,M_\theta|v)$ which by assuming conditional independence can be calculated in the following manner.
\begin{align}
    \pr(M^*,M_\theta|v) = \pr(M^*|v)\pr(M_\theta|v)\\
    \pr(M^*|v) = \pr(v|M^*)\pr(M^*)/\pr(v)
\end{align}
The visualization is presented in Fig \ref{fig:lossindp}.
The maxima of the likelihood function is always achieved in the limit case while the solution for the intersection function is finite and unique. The uniqueness of the solution for intersection can be understood by following the plots as $\theta^*$ changes. 

\textbf{Dependence Assumption:}
To demonstrate another example instead of conditional independence assumption, we can assume that $M^*\subset M_\theta$ and further analyze the behaviour of the objective functions. 
Since $M^*\subset M_\theta$ we can derive the log-likelihood and intersection objective functions using MP theorem and the softmax probability family.
\begin{align}
    &\mathcal{L}_\alpha(M^*|M_\theta) \triangleq -\frac{1}{\alpha}\log\left(\summe{v\in R(V)}{\space}e^{-\alpha(\mlogit_\alpha(v|M_\theta)-\mlogit_\alpha(v|M^*))}\right)\\
    &\mathcal{L}_\alpha(M^*,M_\theta) = \mathcal{L}_\alpha(M^*|M_\theta)+  \mathcal{L}_\alpha(M_\theta)
    \end{align}{}
The behaviour of the objective functions for the Bernoulli example is presented in Fig \ref{fig:LossSubset}.
It is noteworthy that in contrast to the independence assumption the objective functions are well behaved for optimization purposes. In particular in Fig \ref{fig:LossSubset} intersection and likelihood objective functions are concave and the gradients do not vanish in the limit cases. In the case of intersection, the objective function tends to become flat in the open set region between the value of the true underlying parameter and the parameter value of the prior distribution as $\alpha \to \infty$.
\subsection{Application to CNNs}
Here we demonstrate the usage of the MP framework in the context of CNNs and the task of image classification.
In image classification, the set of labeled data $\lbrace z^{(i)}\triangleq(x^{(i)},y^{(i)})\rbrace_{i=1}^N$ is given, where $x^{(i)}\in \mathbb{R}^n$ is the image and $y^{(i)}\in \lbrace 1,\dots, |R(Y)|\rbrace$ is the corresponding label.
To formulate the problem in a probabilistic sense, consider the probability space $(\Omega, \Sigma, \pr)$.
We consider the random variable $Y^{(i)}$ with range $R(Y)$ corresponding to the $i$-th label.
Similarly $X^{(i)}$ is the random variable corresponding to the $i$-th image and $Z^{(i)}$ is the random variable obtained by concatenating the image and the label.
We denote the range of all the image random variables by $R(X)$ and the labels as $R(Y)$.
The CNN model with parameters $\theta\in \mathbb{R}^d$ is the function $f:R(X)\times \mathbb{R}^d\to \Delta^{|R(Y)|}$, where $\Delta^{|R(Y)|}$ is the $|R(Y)|$-dimensional probability simplex.
In the probabilistic sense, the CNN is modeled as $M_\theta\in \Sigma$, where $\pr(y|M_\theta,x^{(i)})$ is the $y$-th component of  $f(x^{(i)},\theta)$.
Since the CNN as a classifier does not model the input distribution, we assume for all $x\in R(X)$ that $\pr(x|M_\theta) = \pr(x)$.
Furthermore we assumed that $\pr(x)$ is equal to the empirical distribution for mathematical convenience.
We define the oracle $M^*$ to characterize the training data, namely 
\begin{align}
\pr(z^{(1)},\dots,z^{(N)}|M^*)=1.
\end{align}
\subsubsection{Log-Likelihood Objective Function}
The log-likelihood objective function, assuming that the oracle and the CNN model are independent conditioned on the observables, can be written as
\begin{align}
    &\logit(M^*|M_\theta)= \nonumber\\
    & \log \left(\summe{z_1,\dots,z_{N}\in R(Z)}{}\pr(M^*|z_1,\dots,z_{N})\pr(z_1,\dots,z_{N}|M_\theta) \right).\label{appeq:cnnll:1}
\end{align}
The CNN model determines the label of any given image independent of the rest of the images. Therefore, we can incorporate the former property as conditional independence of observables
\begin{align}
    \pr(z_1,\dots,z_{N}|M_\theta) = \prod_{i=1}^N \pr(z_i|M_\theta).\label{appeq:cnnll:0}
\end{align}
Considering (\ref{appeq:cnnll:0}) and using the Bayes' rule we can rewrite (\ref{appeq:cnnll:1}) as
\begin{align}
     &\logit(M^*|M_\theta)=\nonumber\\
     & \log \left(\summe{z_1,\dots,z_{N}\in R(Z)}{}\frac{\pr(z_1,\dots,z_N|M^*)\pr(M^*)}{\pr(z_1,\dots,z_N)}\prod_{i=1}^N \pr(z^{(i)}|M_\theta)\right)\label{appeq:cnnll:2}
     \end{align}
     Since $\pr(M^*|z_1,\dots,z_{N})$ is non zero only when $z_1= z^{(1)},\dots , z_N = z^{(N)}$, we can simplify the summation term and obtain
     \begin{align}
     &\logit(M^*|M_\theta)=\nonumber\\
     & \log \left(\frac{\pr(z^{(1)},\dots,z^{(N)}|M^*)\pr(M^*)}{\pr(z^{(1)},\dots,z^{(N)})}\prod_{i=1}^N \pr(z^{(i)}|M_\theta) \right)\label{appeq:cnnll:2}\\
    & =  \left(\summe{i=1}{N} \log \pr(z^{(i)}|M_\theta)\right) +\logit(M^*) - \logit (z^{(1)},\dots,z^{(N)})\\
    &=  \left(\summe{i=1}{N} \log \pr(y^{(i)},x^{(i)}|M_\theta)\right) +\logit(M^*) - \logit (z^{(1)},\dots,z^{(N)})\\
    &=  \left(\summe{i=1}{N} \log \pr(y^{(i)}|M_\theta,x^{(i)})\right) +\left(\summe{i=1}{N} \log \pr(x^{(i)}|M_\theta)\right)\nonumber\\
    &+ \logit(M^*) - \logit (z^{(1)},\dots,z^{(N)}).\label{appeq:cnn:finalll}
\end{align}
The CNN only determines the first term in (\ref{appeq:cnn:finalll}).
Therefore the log-likelihood objective function effectively reduces to the so called Cross Entropy loss in CNNs (without regularization), i.e. the log probability of the correct label given the model and the image.
\begin{figure*}[t]
\center
\includegraphics[scale=0.85]{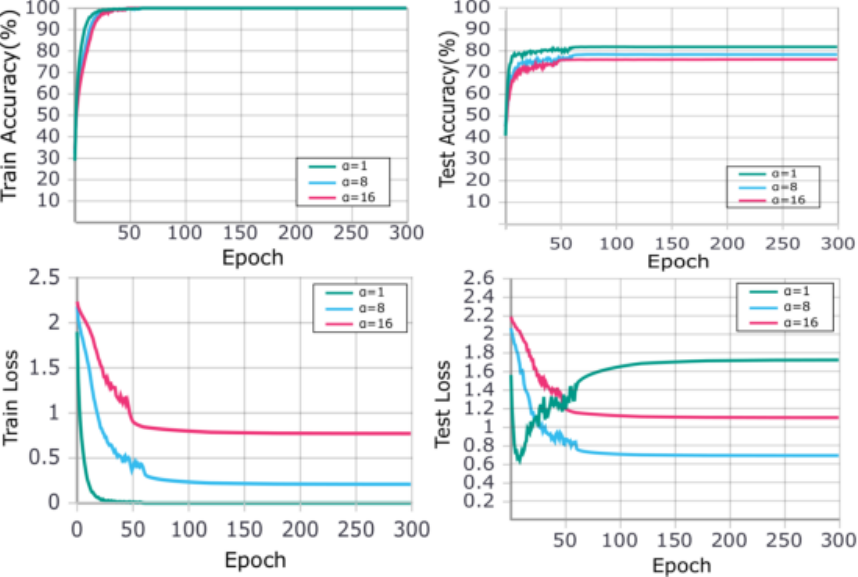}
\caption{The results of training a vanilla CNN using the intersection objective function. The loss represent the cross entropy loss (log-likelihood) commonly used in CNNs. }\label{fig:CNNTrainalpha}
\end{figure*}
\begin{figure*}[t]
\center
\includegraphics[scale=0.85]{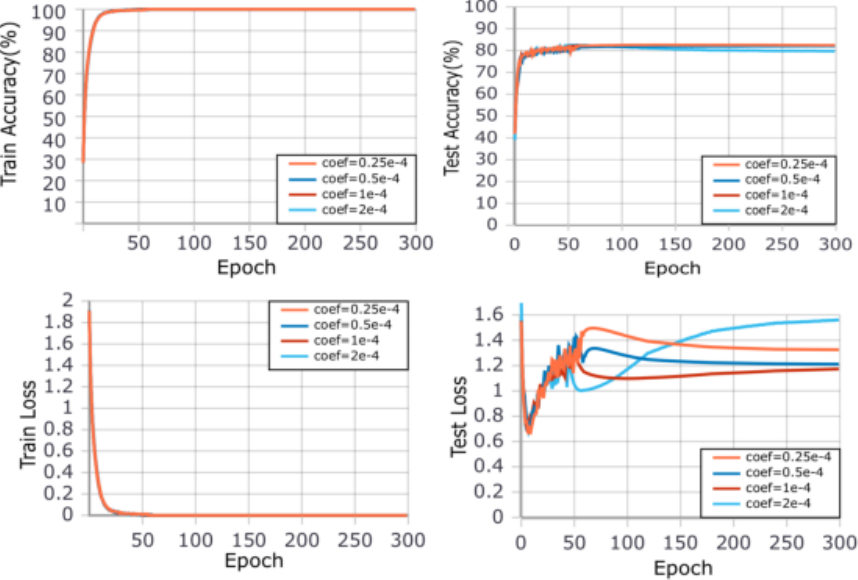}
\caption{The results of training a vanilla CNN with the conventional cross entropy loss and varying coefficient for $\ell_2$ regularization of parameters. }\label{fig:CNNTrainl2}
\end{figure*}
\subsubsection{Intersection Objective Function}
We can use the results obtained from the log-likelihood objective function to calculate the intersection objective function.
The intersection objective function is defined as 
\begin{align}
    \mlogit_\alpha(M^*,M) = \logit(M^*|M) + \mlogit_\alpha(M_\theta)
\end{align}
The log probability of the model can be calculated as 
\begin{align}
    &\mlogit_\alpha(M) = -\frac{1}{\alpha}\log\left(\summe{z_1,\dots,z_N\in R(Z)}{}\left(\frac{\pr(z_1,\dots,z_N)}{\pr(z_1,\dots,z_N|M_\theta)}\right)^{-\alpha} \right).\label{appeq:cnn:probmod}
\end{align}
Note that $z^{(i)}$ and $z_i$ should not be confused. We represented the observations by $z^{(i)}$, while $z_i$ refers to the values iterated in the summation.
Assuming the following property about the prior
\begin{align}
    &\pr(z_1,\dots,z_N) = \prod_{i=1}^N \pr(z_{i}),\\
    &\pr(z_{i}) = \pr(x_{i},y_{i}) = \pr(x_{i})\pr(y_{i})\\
    &\pr(y_{i}) = \frac{1}{R(Y)}
\end{align}
and the property in (\ref{appeq:cnnll:0}) we can write (\ref{appeq:cnn:probmod}) as 
\begin{align}
    \mlogit_\alpha(M) &= -\frac{1}{\alpha}\log\left(\summe{z_1,\dots,z_N\in R(Z)}{}\prod_{i=1}^N\left(\frac{\pr(z_i)}{\pr(z_i|M_\theta)}\right)^{-\alpha} \right)\\
    & = -\frac{1}{\alpha}\log\left(\prod_{i=1}^N\summe{z_i\in R(Z)}{}\left(\frac{\pr(z_i)}{\pr(z_i|M_\theta)}\right)^{-\alpha} \right)\\
    &=  -\frac{1}{\alpha}\summe{i=1}{N}\log\left(\summe{z_i\in R(Z)}{}\left(\frac{\pr(z_i)}{\pr(z_i|M_\theta)}\right)^{-\alpha} \right)\\
    =  -\frac{1}{\alpha}&\summe{i=1}{N}\log\left(\summe{\substack{x_i\in R(X)\\ y_i\in R(Y)} }{}\left(\frac{\pr(x_i)\pr(y_i)}{\pr(y_i|x_i,M_\theta)\pr(x_i|M_\theta)}\right)^{-\alpha} \right)
    \end{align}
Since $\pr(x|M_\theta)=\pr(x)$ and equal to the empirical distribution we can further simplify the log probability of the model into
\begin{align}
    \mlogit_\alpha(M_\theta) = -\frac{1}{\alpha}\summe{i=1}{N}\log\left(\summe{\substack{ y_i\in R(Y)} }{}\pr(y_i|x^{(i)},M_\theta)^{\alpha} \right)\nonumber\\
    + N\log(|R(Y)|).
\end{align}
We conclude that the intersection objective function for the described CNN is 
\begin{align}
    &\mlogit(M_\theta, M^*) =\summe{i=1}{N} \log \pr(y^{(i)}|M_\theta,x^{(i)})\nonumber\\ &-\frac{1}{\alpha}\summe{i=1}{N}\log\left(\summe{\substack{y_i\in R(Y)} }{}\pr(y_i|x^{(i)},M_\theta)^{\alpha} \right) + \textrm{constant}.\label{eq:finalobjectivelosscnn}
\end{align}
Note that by setting $\alpha=1$, the intersection objective function will become equivalent to the log-likelihood objective function by considering that the normalization term appearing when we set $\alpha=1$, is implicit in the \textit{softmax layer}.
We can incorporate the intersection objective function by alternating the softmax layer and continue using the cross entropy (likelihood) loss. 
The softmax layer exponetiate and normalizes the input. We can generalize its functionality by defining the HyperNormalization layer (HN), $\textrm{HN}: \mathbb{R}^m \to \Delta^m$, as 
\begin{align}
    \textrm{HN}(h;\alpha)_i= \frac{e^{h_i}}{\sqrt[\alpha]{\summe{j=1}{m} e^{\alpha h_j}}},
\end{align}
where $\textrm{HN}(h;\alpha)_i$, is the $i$-th component of the output. Note that the functionality of the $\log(\textrm{HN}(h;\alpha)$ and the first two terms of \ref{eq:finalobjectivelosscnn} are similar. Also, by setting $\alpha=1$, the HyperNomalization layer becomes equivalent to the softmax layer.
\begin{table*}[t]
\caption{Quantitative comparison of $\ell_2$ weight regularization versus MP framework. The Regularization column shows the coefficient corresponding to each regularization type. The results of the best case epoch and the final epoch for each measurement is presented in the table. }
\begin{center}
\begin{tabular}{l c c c c c c c c}
\toprule
    \multicolumn{1}{c}{Regularization} & \multicolumn{2}{c}{Test Loss}&\multicolumn{2}{c}{Train Loss}& \multicolumn{2}{c}{Test Accuracy(\%)}&\multicolumn{2}{c}{Train Accuracy(\%)} \\
    \cmidrule(r{2.5pt}l{2.5pt}){1-1} \cmidrule(r{2.5pt}l{2.5pt}){2-3}\cmidrule(r{2.5pt}l{2.5pt}){4-5} \cmidrule(r{2.5pt}l{2.5pt}){6-7}\cmidrule(r{2.5pt}l{2.5pt}){8-9}\\
    &Best&Last&Best&Last&Best&Last&Best&Last\\
    \cmidrule(r{2.5pt}l{2.5pt}){2-2}\cmidrule(r{2.5pt}l{2.5pt}){3-3} \cmidrule(r{2.5pt}l{2.5pt}){4-4}\cmidrule(r{2.5pt}l{2.5pt}){5-5}\cmidrule(r{2.5pt}l{2.5pt}){6-6}\cmidrule(r{2.5pt}l{2.5pt}){7-7} \cmidrule(r{2.5pt}l{2.5pt}){8-8}\cmidrule(r{2.5pt}l{2.5pt}){9-9}\\
    $\alpha=1$&0.645&1.725&1e-6&1e-6&82.16&82.12&100&100\\
    $\alpha=8$&0.695&0.696&0.21&0.21&78.66&78.55&100&100\\
    $\alpha=16$&1.104&1.104&0.771&0.771&76.27&76.27&100&100\\
    $\ell_2=0.25$e-4&0.657&1.3259&7e-6&9e-6&82.58&82.43&100&100\\
    $\ell_2=0.5$e-4&0.693&1.211&2e-5&3e-5&82.29&82.43&100&100\\
    $\ell_2=1$e-4&0.673&1.173&3e-5&1e-4&82.58&82.18&100&100\\
    $\ell_2=2$e-4&0.673&1.651&8e-5&3e-4&82.58&79.63&100&100\\
    \bottomrule
\end{tabular}\label{table:results}
\end{center}
\end{table*}

We experimented the effects of using the intersection objective function with varying coefficients on CNNs.
The experiments were done on the CIFAR10 image dataset \cite{krizhevsky2009learning} labeled into 10 classes with 50000 training samples and 10000 test samples.
The cross entropy loss and accuracy for the performance of the CNN on the test set and the training set were measured.
Using the intersection objective function, the regularization effect is back-propagated through the network from the HyperNormalization layer.
Using $\alpha=1$, the objective function is similar to the cross entropy loss without regularization.
The effects of training a vanilla CNN (without BatchNorm, Dropout) using the intersection objective function with varying value of $\alpha$ is presented in Fig \ref{fig:CNNTrainalpha}. To compare the results with the conventional cross entropy loss and $\ell_2$ regularization (also known as weight decay) please refer to Fig \ref{fig:CNNTrainl2}.
The $\ell_2$ regularization loss could be obtained by following the conventional Bayesian framework, namely considering the parameters as a random variable. The $\ell_2$ regularization loss appears when considering the normal distribution as the prior distribution of the parameter random variables.
Also, the coefficient of the regularization loss is determined by the variance of the normal distribution assumed on the parameters.
The goal of the conventional Bayesian framework is to obtain the maximum a posteriori (MAP) estimate of the parameters.

By comparing the test loss among different values of $\alpha$, we can see that overfitting is prevented by increasing the value of $\alpha$, while affecting the test accuracy minimally.
Also, CNNs trained with $\alpha>1$, generalize better than the $\ell_2$ regularized network, when considering the test loss value.
In Fig \ref{fig:CNNTrainl2} the test loss of $\ell_2$ regularized networks is increased as the training continues. 
We can make a similar observation in Fig \ref{fig:CNNTrainalpha} for the un-regularized network with $\alpha=1$.
In Fig \ref{fig:CNNTrainl2} we can see that increasing the the coefficient of the $\ell_2$ regularization does not prevent the test loss increase.
The network with the largest $\ell_2$ coefficient achieves the largest error in the final epoch.
In contrast with the $\ell_2$ regularized networks, the test loss of the hypernormalized networks converges to a stable minima.
On the other hand $\ell_2$ regularization is more successful in generalization of accuracy but performs worse in generalization of loss.
The quantitative results of the experiments are shown in Table \ref{table:results}.
Table \ref{table:results} demonstrates the accuracy and loss of the tested methods during the best case epoch and the final epoch. 
Table \ref{table:results} shows that by increasing the value of $\alpha$, the gap between the train loss and test loss decreases.
Also, the the gap between the best and last epoch test loss decreases as $\alpha$ increases, which shows the stability of the generalization process. 
In the case of $\ell_2$ regularization, increasing the coefficient, does not affect the gap of training and test loss and mostly does not impact the test accuracy.

Note that the accuracy is a surrogate objective function and although useful, does not fully characterize the performance.
We hypothesize that the reason that test accuracy is not improved in our example (with intersection objective function), is because of the depth of the network.
The regularization effect may disappear similar to the effect of gradient vanishing.
Including more random variables in the middle layers and considering them in the objective function may improve the test accuracy results.
\section{Discussion}
In the current paper, we have presented and proved the MP theorem.
MP theorem quantifies the upper bound for probabilities of events by having their respective conditional distributions.
We showed that considering the upper bound as the probability of the event agrees with the existing definitions and extends our ability to quantify the probability of uncertain observations.
The MP theorem was used to define models, quantify their probability measure, and develop objective functions; resulting in the MP framework for probabilistic learning.
MP framework treats the parameterized model and the oracle as events in the underlying probability space.
Considering the underlying space enables using set operations to represent similarities between two events and construct objective functions.
We used the example of using likelihood and intersection as the objective function and showed the sufficient conditions of their solutions.

MP framework requires a prior distribution over the observable random variable while the choice of prior is not determined by the framework.
Thereby existing principles for determining priors need to be used in the framework e.g. MAXENT and Laplace's Principle.
The usage of the MP theorem is helpful since the prior only needs to be determined over the observable random variable.
As a corollary, the complexity of developing the prior distribution will be relative to the complexity of observable random variable.
It is common that we have prior information about the observables rather than the hidden random variables.
In such cases determining the prior over observable random variable are more convenient than the hidden random variables.
Our framework is developed for finite-range random variables, and the generalization to the continuous case is left for future works.

MP framework allows development and analysis of objective functions other than likelihood and intersection, e.g. Symmetric Difference of events.
The motivation to explore other objective functions is to avoid the trivial solutions that are inherently present in likelihood and intersection.
Note that the objective functions discussed so far have an important feature.
The gradient vectors of the log-likelihood and log-probability of the model are probability vectors. 
This property enables us to approximate the gradients by Monte Carlo approximation; \textbf{(i)} connecting this framework with stochastic optimization theory and techniques \textbf{ (ii)} enabling black-box optimization of probabilistic models by sampling from their corresponding likelihood functions \textbf{(iii)} scalability to factorized random variables having an exponentially large number of states. 
In the MP framework, the formal treatment of probabilistic models if coupled with the ability to optimize large scale models helps probabilistic models to move toward an axiomatic and practical approach to large scale machine learning.


\section*{Appendix}
\setcounter{theorem}{0}
\setcounter{proposition}{0}
\setcounter{corollary}{0}
\subsection{Proofs}
\begin{theorem}\label{thrm:Master}
Consider the probability space $(\Omega,\Sigma,\pr)$, the random variable $V$ with finite range $R(V)$ and $P_V(.)$ the probability distribution of $V$. For any event $\sigma \in \Sigma$ with the conditional distribution $P_{V|\sigma}(.)$ the following holds,

\begin{align}
    \pr(\sigma) \leq\underset{v\in R(V)}{\inf} \left \{ \frac{P_V(v)}{P_{V|\sigma}(v) }\right \} \triangleq \mathcal{P}(\sigma \succ V)
\end{align}
$\mathcal{P}(\sigma \succ V)$ is read as maximum probability of $\sigma$ observed by $V$ and $-\mathcal{L}(\sigma \succ V)= -\log\left(\mpr(\sigma \succ V)\right)$ is read as the minimum information in $\sigma$ observed by $V$.
\end{theorem}
\begin{proof}
 $\forall \sigma \in \Sigma$ and $\forall v \in R(V)$ the following is true
\begin{align}
    \pr(v) = \pr(v|\sigma)\pr(\sigma)+\pr(v|\Bar{\sigma})\pr(\Bar{\sigma})
\end{align}
since $\pr(v|\Bar{\sigma})\pr(\Bar{\sigma})\geq 0$ then
\begin{align}
    &\pr(v)-\pr(v|\sigma)\pr(\sigma) \geq 0\label{proof:zerovalue}\\
    &\pr(\sigma)\leq \frac{\pr(v)}{\pr(v|\sigma)}.\label{proof:zerovalue2}
\end{align}
Since (\ref{proof:zerovalue2}) is true for all $v$ such that $\pr(v|\sigma)\neq 0$, then the following holds
\begin{align}
    \pr(\sigma) \leq \underset{v\in R(V)}{\inf} \left \{ \frac{\pr(v)}{\pr(v|\sigma) }\right \}
\end{align}

\end{proof}
\begin{theorem}\label{thrm:ExtendedInfo}
For any random variable $W$ that extends $V$, the following inequality holds
\begin{align}
\mathcal{P}(\sigma\succ W) \leq \mathcal{P}(\sigma \succ V)
\end{align}
\end{theorem}
\begin{proof} 
Since $W$  extends $V$ and $\forall w\in R(W)$ and $\forall v \in R(V)$, $W^{-1}(w)$ is either the subset of $V^{-1}(v)$ or does not have intersection. Similarly $\pr(w,v)$ is either $0$ or $\pr(w,v)=\pr(w)$. Defining the set $R_v(W)\triangleq\lbrace w'|w'\in R(W),W^{-1}(w')\subset V^{-1}(v)\rbrace $, we can write
\begin{align}
    \pr(v)=\summe{w'\in R(W)}{}\pr(w',v)= \summe{w'\in R_v(W)}{}\pr(w',v)\nonumber\\ =\summe{w'\in R_v(W)}{}\pr(w') \label{proof:th4_1}
\end{align}
and similarly
\begin{align}
    \pr(v|\sigma) = \summe{w'\in R(W)}{}\pr(w',v|\sigma) = \summe{w'\in R_v(W)}{}\pr(w'|\sigma).\label{proof:th4_2}
\end{align}
We know that $\pr(w)\geq \pr(w|\sigma)\pr(\sigma)$. Using Theorem \ref{thrm:Master}, $\pr(\sigma)$ can be replaced with $\mathcal{P}(\sigma \succ W)$ and write
\begin{align}
    \pr(w)\geq \pr(w|\sigma) \mathcal{P}(\sigma \succ W).\label{proof:th4_new1}
\end{align}{}
Using (\ref{proof:th4_1},\ref{proof:th4_new1}) we can see that
\begin{align}
    \pr(v) = \summe{w'\in R_v(W)}{}\pr(w')\geq  \summe{w'\in R_v(W)}{}\pr(w'|\sigma)\mathcal{P}(\sigma \succ W)\nonumber\\
    = \mathcal{P}(\sigma \succ W)\summe{w'\in R_v(W)}{}\pr(w'|\sigma).
\end{align}{}
Substituting (\ref{proof:th4_2}) we conclude that
\begin{align}
        &\pr(v) \geq  \pr(v|\sigma)\mathcal{P}(\sigma \succ W),\\
        &\frac{\pr(v)}{\pr(v|\sigma)} \geq  \mathcal{P}(\sigma \succ W).\label{proof:th4_3}
\end{align}
Since (\ref{proof:th4_3}) is valid for all possible outcomes of $V$, the inequality is valid for the minimum of the left-hand side of (\ref{proof:th4_3}). Consequently
\begin{align}
    \underset{v'\in R(V)}{\inf} \left \{ \frac{P_V(v')}{P_{V|\sigma}(v') }\right \} \geq  \mathcal{P}(\sigma \succ W)\\
     \mathcal{P}(\sigma \succ V) \geq  \mathcal{P}(\sigma \succ W)
\end{align}{}
\end{proof}
\begin{corollary}\label{remark:concatinfo}
For any random variable $V$ and $Z$ with concatenation $H=(V,Z)$ the following holds
\begin{align}
\mathcal{P}(\sigma\succ H) \leq \mathcal{P}(\sigma \succ V)
\end{align}
\end{corollary}
\begin{proof}
the proof follows from Theorem \ref{thrm:ExtendedInfo}, if $H$ is extending $V$.
By definition of $(V,Z)$, for all outcomes $h \in R(H)$, $H^{-1}(h) = V^{-1}(v_h)\cap Z^{-1}(z_h)$ for some $z_h\in R(Z),v_h \in R(V)$ such that $h=(v_h,z_h)$.
Therefore every outcome of $H$ is either a subset of some partition induced by  $V$ or does not have intersection with it.
Since $H$ extends $V$, we can directly use Theorem \ref{thrm:ExtendedInfo} to finalize the proof.
\end{proof}
\begin{proposition}\label{def:maxprobdef}
Probability of an outcome $v$ of random variable $V$ in the sense of maximum probability defined as
\begin{align}
    P^*_V(v) \triangleq \sup\left\lbrace\pr(\sigma) | \sigma \in \Sigma , \forall \omega \in \sigma, V(\omega) = v  \right\rbrace,\quad  v \in R(V)
\end{align}
has the following property
\begin{align}
P^*_V(v) = P_V(v)\quad \forall v \in R(V).    
\end{align}
\end{proposition}
\begin{proof}
Consider the set $\Sigma_v= \lbrace \sigma | \sigma \in \Sigma , \forall \omega \in \sigma_V, V(\omega) =V\rbrace$
for all $\sigma_v\in \Sigma_v$ we can say that $\sigma_v\subset V^{-1}(v))$ using the definition of $V^{-1}$.
by the monotonicity of probability we can conclude $\pr(\sigma_v)\leq \pr(V^{-1}(v)$ for all $\sigma_v$. Therefore,
\begin{align}
 \pr(V^{-1}(v)) =\sup\left\lbrace\pr(\sigma) | \sigma \in \Sigma , \forall \omega \in \sigma, V(\omega) = v  \right\rbrace
\end{align}
for all $v\in R(V)$.
\end{proof}
\begin{corollary}\label{corollary:mprob_eqv}
Considering Definition \ref{def:classicvar} and Theorem \ref{thrm:Master}, given some set $\sigma_v\in \Sigma$ where $P_{V|\sigma_v}(v)=1$,
\begin{align}
P_V(v) = \mathcal{P}(\sigma_v \succ V)    
\end{align}
\end{corollary}
\begin{proof}
we can write
\begin{align}
   \mathcal{P}(\sigma_v \succ V)&=  \underset{v'\in R(V)}{\inf} \left \{ \frac{P_V(v')}{P_{V|\sigma_v}(v') }\right \}
\end{align}
Note that outcomes with zero probability in $P_{V|\sigma_v}$ do not pose constraints on probability of $\sigma_v$ (refer to (\ref{proof:zerovalue})).
Since $P_V$ is zero for every $v'\neq v$ we get
\begin{align}
    \mathcal{P}(\sigma_v \succ V)&= \underset{}{\inf} \left \{ \frac{P_V(v)}{1}\right \} = P_V(v).  
\end{align}

\end{proof}
\begin{lemma}\label{lemma:prop6}
The following inequality holds
\begin{align}
&\underset{i}{\sup}\left\lbrace a_i\right\rbrace \leq \frac{1}{\alpha}\log \left(\summe{i}{}e^{\alpha a_i} \right)\\
&\underset{i}{\inf}\left\lbrace a_i\right\rbrace \geq -\frac{1}{\alpha}\log \left(\summe{i}{}e^{-\alpha a_i} \right)
, \quad a_i\in \mathbb{R}, \alpha >0
\end{align}
where the equality holds as $\alpha \to \infty$.
\end{lemma}
\begin{proof}
\textbf{(Lemma)} Let us denote $\underset{i}{\sup}\left\lbrace  a_i\right\rbrace = a^*$
\begin{align}
    \frac{1}{\alpha}\log \left(\summe{i}{}e^{\alpha a_i} \right)= \frac{1}{\alpha}\log \left(\summe{i}{}e^{\alpha (a_i-a^*+a^*)} \right)=\nonumber\\
    a^* +\frac{1}{\alpha}\log \left(\summe{i}{}e^{\alpha (a_i-a^*)} \right).\label{lemma:prop6_1}
\end{align}
since there exist at least an element in the set having the same value as supremum of the set, $\log \left(\summe{i}{}e^{\alpha (a_i-a^*)} \right) = \log(1+ C)\geq 0$ for some positive constant $C$ and therefore positive. Since $\alpha >0$, (\ref{lemma:prop6_1}) is greater than the supremum
\begin{align}
    &\frac{1}{\alpha}\log \left(\summe{i}{}e^{\alpha a_i} \right) =\nonumber\\
    &\underset{i}{\sup}\left\lbrace a_i\right\rbrace +\frac{1}{\alpha}\log \left(\summe{i}{}e^{\alpha (a_i-a^*)} \right) \geq \underset{i}{\sup}\left\lbrace a_i\right\rbrace
\end{align}
In the limit case since 
\begin{align}
    \underset{\alpha \to \infty}{\lim }a^* +\frac{1}{\alpha}\log &\left(\summe{i}{}e^{\alpha (a_i-a^*)} \right) = \nonumber\\
    &a^* + \underset{\alpha \to \infty}{\lim} \frac{1}{\alpha}\log \left(\summe{i}{}e^{\alpha (a_i-a^*)} \right)
\end{align}
since $(a_i-a^*)\leq 0$, then 
\begin{align}
    a^*+\underset{\alpha \to \infty}{\lim } \frac{1}{\alpha}\log \left(\summe{i}{}e^{\alpha (a_i-a^*)} \right) = a^*
\end{align}
We can use the fact that $-\underset{i}{\sup}\left\lbrace -a_i\right\rbrace = \underset{i}{\inf}\left\lbrace a_i\right\rbrace$ to prove the lemma for the infimum case.
\end{proof}
\begin{proposition}\label{prop:alphafamily}
For the softmax probability family of functions $\mathcal{P}_\alpha$ defined as
\begin{align}
    &\mathcal{P}_\alpha(\sigma\succ V) \triangleq \left(\summe{v\in R(V)}{\space}\left(\frac{P_V(v)}{P_{V|\sigma}(v)}\right)^{-\alpha}\right)^{-\frac{1}{\alpha}},\quad \alpha >0,
    \end{align}
    the following is true,
    \begin{align}
    &\mathcal{P}_\alpha(\sigma\succ V)\leq \mathcal{P}(\sigma \succ V) ,\quad \forall \alpha> 0,\forall \sigma \in \Sigma,
\end{align}
and the equality holds as $\alpha \to +\infty$. The softmin information family is defined as
\begin{align}
-\mathcal{L}_{\alpha}(\sigma\succ V) \triangleq -\log\left(\mathcal{P}_\alpha(\sigma\succ V) \right).    
\end{align}
\end{proposition}
\begin{proof}

By defintion of $\mlogit$ we have
\begin{align}
    \mlogit(\sigma\prec V) = \underset{v\in R(V)}{\inf} \left \{ {\logit(v) -\logit(v|\sigma)}\right \}.
\end{align}
By directly using the inequality proved in Lemma \ref{lemma:prop6} we can write
\begin{align}
    \underset{v\in R(V)}{\inf} \left \{ {\logit(v) -\logit(v|\sigma)}\right \} &\geq -\frac{1}{\alpha}\log\left( e^{-\alpha(\logit(v) -\logit(v|\sigma))} \right)\\
    \mathcal{L}(\sigma \succ V) &\geq \mathcal{L}_\alpha(\sigma \succ V)
\end{align}
and exponentiating the above inequality we get
\begin{align}
    \mathcal{P}(\sigma \succ V) \geq \mathcal{P}_\alpha(\sigma \succ V)
\end{align}
\end{proof}
\subsection{Log Likelihood Derivation}
The sufficient condition for the global maximizer of $\logit(M^*|M)$ is brought in the following proposition.
\begin{proposition}
Consider the function $\logit(M^*|M_\theta)$ and the following maximization problem
\begin{align}
    \theta^* = \underset{\theta \in \mathbb{R}^d}{\arg \max }\{ \logit(M^*|M_\theta) \}.\label{appeq:ll2}
\end{align}
If $M$ and $M^*$ are independent conditioned on $V$, The sufficient condition for $\theta$ to be the global maximizer of (\ref{appeq:ll2}) is
\begin{align}
    &\pr(v'|M_\theta^*) =0 , \nonumber\\
    &\forall v',\pr(v'|M^*)\neq \underset{v'' \in R(V)}{\max} \pr(v''|M^*)
\end{align}
\end{proposition}
\begin{proof}
We write the log-likelihood as
\begin{align}
\logit(M^*|M) = \log \Bigg(\summe{v\in R(V)}{}\pr(M^*|v,M)\pr(v|M)\Bigg)\nonumber\\
-\log\Bigg(\summe{v'\in R(V)}{}\pr(v'|M)\Bigg)
\end{align}
\begin{align}
=\log \Bigg(&\summe{v\in R(V)}{}\frac{\pr(v|M^*)\pr(v|M)}{\pr(v)}\Bigg)\nonumber\\
&-\log\Bigg(\summe{v'\in R(V)}{}\pr(v'|M)\Bigg)\label{appeq:ll1} + \log(\pr(M^*)).
\end{align}
The second and third term in the right hand side of (\ref{appeq:ll1}) are constants. The first term is logarithm of a convex combination of $\pr(v|M^*)/\pr(v)$. We can bound the first term
\begin{align}
     \log \Bigg(\summe{v\in R(V)}{}&\frac{\pr(v|M^*)}{\pr(v)}\pr(v|M)\Bigg) \leq \log\left(\underset{v\in R(V)}{\sup}\left\lbrace \frac{\pr(v|M^*)}{\pr(v)}  \right\rbrace \right).\label{appeq:ll2}
\end{align}
We define $R_{\textrm{max}}(V)\subset R(V)$ as $R_{\textrm{max}}= \left\lbrace v|v\in R(V), \pr(v|M^*)\neq \underset{v\in R(V)}{\sup}\left\lbrace \frac{\pr(v|M^*)}{\pr(v)}  \right\rbrace\right\rbrace$.
We can check that the equality in (\ref{appeq:ll2}) holds for some $\theta'$, if the condition
\begin{align}
    \pr(v |M_{\theta'}) =0 \iff \forall v\in R(V) , v\notin R_{\textrm{max}}(V)\label{appeq:llsufficient}
\end{align}
is true.

For $\theta'$, (\ref{appeq:ll2}) reduces to 
\begin{align}
    \log\Bigg(\summe{v\in R_{\textrm{max}}(V)}{}\frac{\pr(v|M^*)}{\pr(v)}\pr(v|M_{\theta'})\Bigg) =\nonumber\\
    \log\Bigg(\summe{v\in R_{\textrm{max}}(V)}{}\underset{v\in R(V)}{\sup}\left\lbrace \frac{\pr(v|M^*)}{\pr(v)}  \right\rbrace\pr(v|M_{\theta'})\Bigg)\nonumber\\
    = \log\left(\underset{v\in R(V)}{\sup}\left\lbrace \frac{\pr(v|M^*)}{\pr(v)}  \right\rbrace \right).
\end{align}

\end{proof}
$\theta'$ achieves the upper bound and therefore the condition in (\ref{appeq:llsufficient}) is a sufficient condition for the solution of (\ref{appeq:ll2}).
\subsection{Intersection Derivation}
We start by rewriting the conditional independence assumption and the uniform prior 
\begin{align}
    &\pr(M,M^*|v) = \pr(M|v)\pr(M^*|v),\label{appeq:intersect:assume1}\\
    &\pr(v) = \frac{1}{|R(V)|}.\label{appeq:intersect:assume2}
\end{align}
Remember that the sufficient condition for the solution of intersection objective function was
\begin{align}
     \pr(v|M_\theta,M^*) =  \pr_\alpha(v|M_\theta)\label{appeq:intersect:condition}.
 \end{align}
Using the assumptions in (\ref{appeq:intersect:assume1},\ref{appeq:intersect:assume2}), we can represent the left hand side of (\ref{appeq:intersect:condition}), i.e. the posterior, as
 \begin{align}
     \frac{\pr(M_\theta,M^*|v)\pr(v)}{\pr(M_\theta,M^*)}=\frac{\pr(M_\theta|v)\pr(M^*|v)\pr(v)}{\pr(M_\theta,M^*)}=\nonumber\\
     \frac{\pr(v|M_\theta)\pr(v|M^*)\pr(M_\theta)\pr(M^*)}{\pr(M_\theta,M^*)\pr(v)}.\label{appeq:intersect:finalform} 
\end{align}
Having (\ref{appeq:intersect:assume2}), we can see that (\ref{appeq:intersect:finalform}) is element-wise multiplication of two probability vectors followed by normalization, therefore the condition in (\ref{appeq:intersect:condition}) reduces to
\begin{align}
    \frac{\pr(v|M_\theta)\pr(v|M^*)}{C} = \pr_\alpha(v|M),\label{appeq:intersection:indpcondition}\\
    C = \summe{v' \in R(V)}{}\pr(v|M_\theta)\pr(v|M^*).
\end{align}
\subsubsection{Setting $\alpha=2$}
In the case of $\alpha =2 $, (\ref{appeq:intersection:indpcondition}) is 
\begin{align}
     \frac{\pr(v|M_\theta)\pr(v|M^*)}{\summe{v' \in R(V)}{}\pr(v|M_\theta)\pr(v|M^*)} = \frac{\pr(v|M_\theta)\pr(v|M_\theta)}{\summe{v'\in R(V)}{}\pr(v|M_\theta)\pr(v|M_\theta)}.
\end{align}
We can show that for all $\theta^*\in \mathbb{R}^d$ such that
\begin{equation}
    \pr(v|M_{\theta^*})= \pr(v|M^*),\label{appeq:intersection:a2assum}
\end{equation} the condition in (\ref{appeq:intersection:indpcondition}) is satisfied. By plugging in $\theta^*$ in (\ref{appeq:intersection:indpcondition}) we get
\begin{align}
     \frac{\pr(v|M_{\theta^*})\pr(v|M^*)}{\summe{v' \in R(V)}{}\pr(v|M_{\theta^*})\pr(v|M^*)} &= \frac{\pr(v|M_{\theta^*})\pr(v|M_{\theta^*})}{\summe{v'\in R(V)}{}\pr(v|M_{\theta^*})\pr(v|M_{\theta^*})}.
\end{align}
If we use (\ref{appeq:intersection:a2assum}), we can see that the condition is satisfied. Therefore in the case of $\alpha=2$, the distribution of the oracle is a solution.

\section*{Acknowledgment}

\bibliographystyle{abbrvnat}
\bibliography{maxprob.bib}

\begin{IEEEbiography}[{\includegraphics[width=1in,height=1.25in,clip,keepaspectratio]{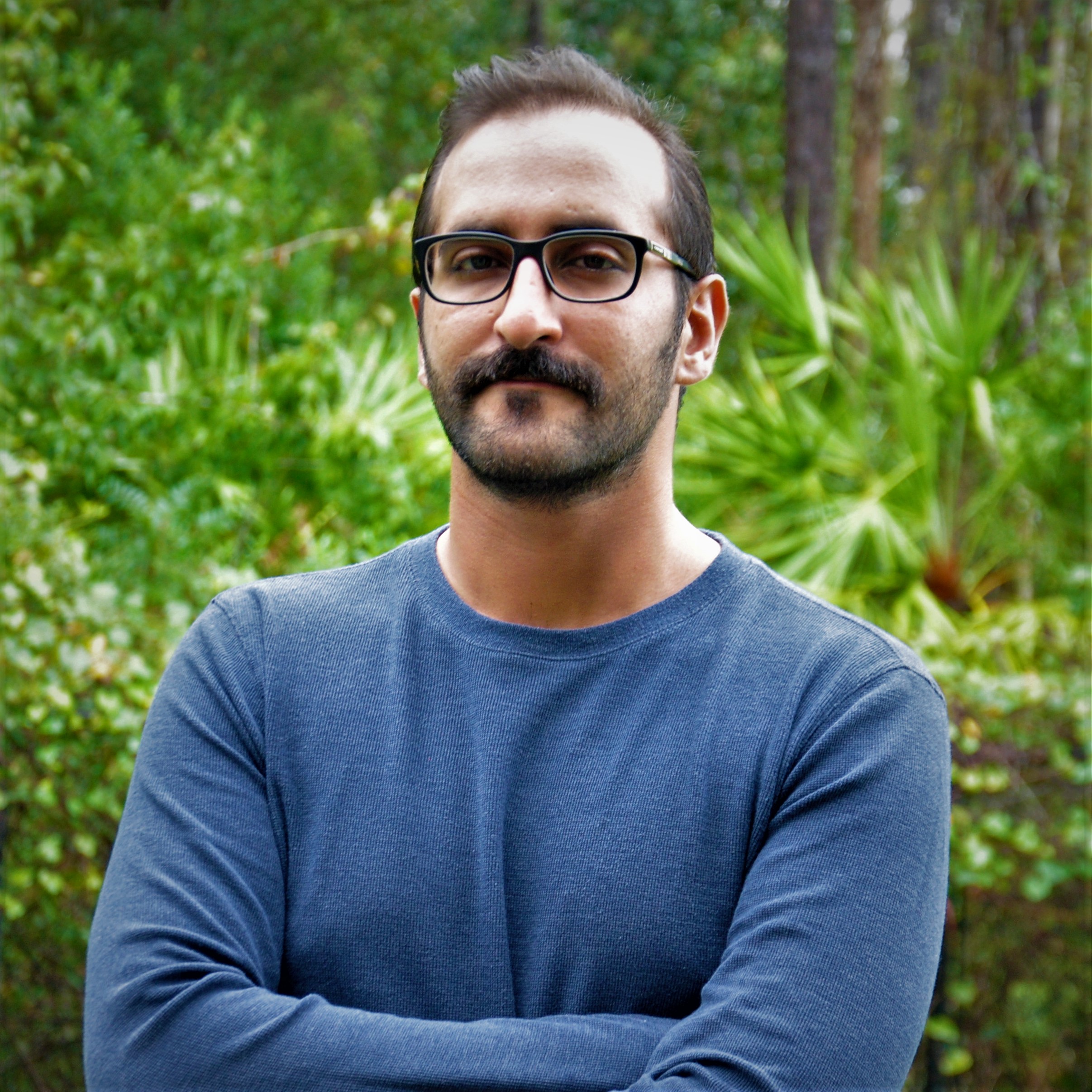}}]{Amir Emad Marvasti } joined the computer science department as a PhD student in 2014. Currently, he works at Computational Imaging Lab(CIL) under the supervision of Professor Hassan Foroosh. He received his B.S. in Computer Engineer from Sharif University of Technology, Tehran, Iran. 
His research interests are Probabilistic Machine Learning and Probability Theory. At CIL, he is focusing on the formalization of prior knowledge, regularization of probabilistic models, and development of finite-state probabilistic models.
\end{IEEEbiography}
\begin{IEEEbiography}[{\includegraphics[width=1in,height=1.25in,clip,keepaspectratio]{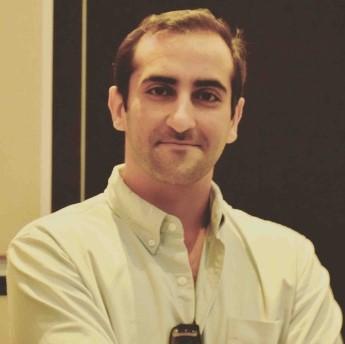}}]{Ehsan Emad Marvasti } received the B.S. degree in computer engineering from the Sharif University of Technology, Tehran, Iran, 2014. He is currently working toward the Ph.D. degree in computer science at the Department of Computer Science, University of Central Florida, Orlando, FL, USA. He has been a member of Image Processing Lab at Sharif University of Technology from 2012-2014 and Computational Imaging Lab (CIL) at University of Central Florida from 2014-2017. Currently he is a member of Connected and Autonomous Vehicle Research Lab (CAVREL). His research interests include vehicular cooperative perception and cognition, cooperative sensor fusion and machine learning.
\end{IEEEbiography}

\begin{IEEEbiography}[{\includegraphics[width=1in,height=1.25in,clip,keepaspectratio]{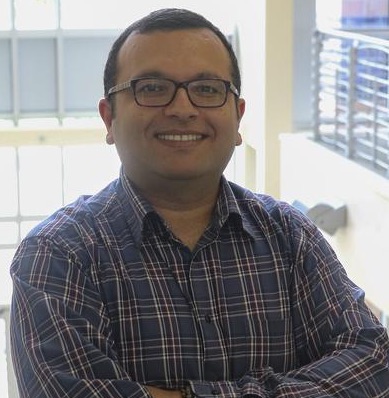}}]{Dr.Ulas Bagci} is a faculty member at the Center for Research in Computer Vision (CRCV), and the SAIC Chair Professor in University of Central Florida (UCF). His research interests are artificial intelligence, machine learning and their applications in biomedical and clinical imaging. Prof. Bagci has more than 200 peer-reviewed articles in related topics. Previously, he was a staff scientist and lab co-manager at the National Institutes of Health’s radiology and imaging sciences department, center for infectious disease imaging. Dr. Bagci holds two NIH R01 grants (Principal Investigator) and serve as an external committee member of AIR (artificial intelligence resource) at the NIH. Dr. Bagci serves as an area chair for MICCAI for several years and he is an associate editor of top-tier journals in his fields such as IEEE Trans. on Medical Imaging and Medical Physics, and editorial board member of Medical Image Analysis.
\end{IEEEbiography}
\begin{IEEEbiography}[{\includegraphics[width=1in,height=1.25in,clip,keepaspectratio]{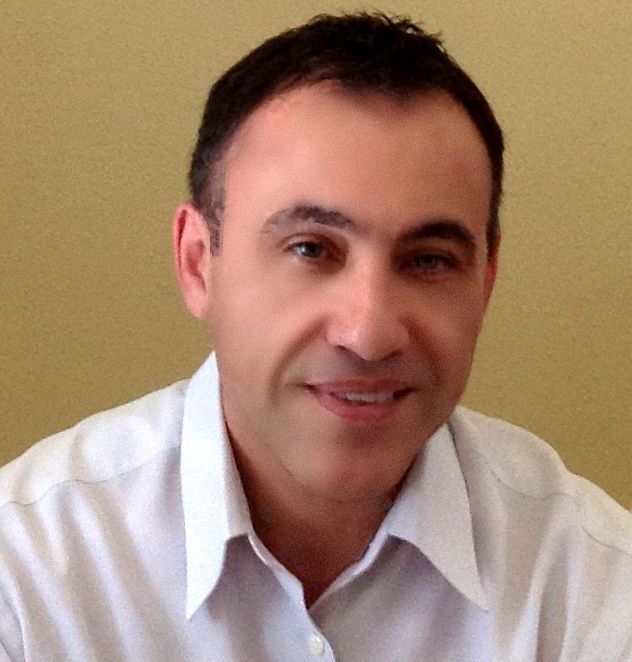}}]{Hassan Foroosh}(Senior Member, IEEE) is currently a CAE Link Professor of computer science with the University of Central Florida (UCF), Orlando, FL, USA. He has authored or coauthored over 160 peer-reviewed scientific articles in the areas of computer vision, image processing, and machine learning. He received the Piero Zamperoni Award from the International Association of Pattern Recognition (IAPR), in 2004, the Best Scientific Paper Award from IAPR-International Conference on Pattern Recognition (IAPR-ICPR), in 2008, and the Best Paper Award from the IEEE International Conference on Image Processing (ICIP), in 2018. He is also the Principal Investigator and the Lead of the Science Data Center of the NASA GOLD Mission that launched a satellite into Earth’s geo-stationary orbit to study the space weather using ultraviolet imaging, in 2018. He has been serving on the Editorial Boards and Organizing Committees of various IEEE Transactions, conferences, and working groups.
\end{IEEEbiography}

\end{document}